\pdfoutput=1
\documentclass[11pt,twoside]{article}
\newcommand{\TitleRunning}{A Unified Analysis Method for Online Optimization in Normed Vector Space}
\newcommand{\AuthorRunning}{Qing-xin Meng and Jian-wei Liu}
\newcommand{\Keywords}{optimistic online learning, online convex optimization, online monotone optimization, dynamic regret, static regret, generalized Bregman divergence, $\phi$-convex}
\usepackage{arxiv}
\papertitle{A Unified Analysis Method for Online Optimization in Normed Vector Space}

\author{\sc Qing-xin Meng
\orcid{0000-0003-4014-7405}
\\\small\email{qingxin6174@gmail.com}
\and\sc Jian-wei Liu
\\\small\email{liujw@cup.edu.cn}
}

\addr{Department of Automation, College of Information Science and Engineering \\
China University of Petroleum-Beijing (CUP), Beijing 102249, China}
\newcommand{\si}{S\textup{-I} }
\newcommand{\sii}{S\textup{-II} }
\DeclareMathOperator*{\epi}{\operatorname{epi}}
\DeclareMathOperator*{\gra}{\operatorname{gra}}
\DeclareMathOperator*{\dom}{\operatorname{dom}}
\DeclareMathOperator*{\regret}{\operatorname{regret}}
\newcommand{\regretn}[1]{\sideset{}{^{#1}}\regret}
\begin{document}
\maketitle\thispagestyle{firstpage}
%%%%%%%%%%%%%%%%%%%%%%%%%%%%%%%%%%%%%%

%%%%%%%%%%%%%%%%%%%%%%%%%%%%%%%%%%%%%%
\begin{abstract}\noindent
This paper studies online optimization from a high-level unified theoretical perspective. 
We not only generalize both Optimistic-DA and Optimistic-MD in normed vector space, but also unify their analysis methods for dynamic regret. 
Regret bounds are the tightest possible due to the introduction of $\phi$-convex. 
As instantiations, regret bounds of normalized exponentiated subgradient and greedy/lazy projection are better than the currently known optimal results. 
By replacing losses of online game with monotone operators, and extending the definition of regret, namely $\regretn{n}$, we extend online convex optimization to online monotone optimization. 
\end{abstract}
%%%%%%%%%%%%%%%%%%%%%%%%%%%%%%%%%%%%%%

%%%%%%%%%%%%%% BODY TEXT %%%%%%%%%%%%%

%%%%%%%%%%%%%%%%%%%%%%%%%%%%%%%%%%%%%%
\section{Introduction}
%%%%%%%%%%%%%%%%%%%%%%%%%%%%%%%%%%%%%%

%%%%%%%%%%%%%%%%%%%%%%%%%%%%%%%%%%%%%%
The online convex optimization problem introduced by \cite{zinkevich2003online} can be regarded as a repeated game between the learner and the adversary (environment). 
At round $t$, the learner chooses a map $x_t$ from a hypothesis class $C$ as prediction, and the adversary feeds back a convex loss function $\varphi_t$, then the learner suffers an instantaneous loss $\varphi_t\left(x_t\right)$. 
The goal is to minimize 
\begin{equation}
\label{dynamic-regret}
\regret \left(z_1, z_2, \cdots, z_T\right)\coloneqq\sum_{t=1}^T \varphi_t\left(x_t\right)-\sum_{t=1}^T \varphi_t\left(z_t\right),
\end{equation}
where $z_t\in C$ represents an arbitrary reference strategy in round $t$, and $T$ is the number of rounds. 
\cref{dynamic-regret} is called the dynamic regret. 
For static regret, which often appears in the literature \citep[e.g.,][]{bianchi2006prediction, shwartz2012online, mcmahan2017survey, hazan2019introduction, orabona2020modern}, simply let $z_t\equiv z$. 
Although dynamic regret has attracted widespread attention recently \citep[e.g.,][]{hall2013dynamical, jadbabaie2015online, mokhtari2016online, zhang2018adaptive,  campolongo2021closer, kalhan2021dynamic}, it still lacks systematic research.
%%%%%%%%%%%%%%%%%%%%%%%%%%%%%%%%%%%%%%

%%%%%%%%%%%%%%%%%%%%%%%%%%%%%%%%%%%%%%
Generally, strategies for online optimization fall into two major families,  Mirror Descent~(MD) and Follow The Regularized Leader~(FTRL). 
MD was introduced by \cite{beck2003mirror}, and a special form of MD can be traced back to  \cite{nemirovskii1983problem}. 
FTRL was introduced by \cite{abernethy2008competing}, and its core ideas can be traced back to \cite{shwartz2006Online}. 
FTRL with surrogate linearized losses is also called Dual Averaging \citep[DA][]{nesterov2009primal}. 
An algorithm is optimistic, if the prediction of the impending loss is added to the update rule. 
The Optimistic-MD was proposed by \cite{chiang2012online}, and extended by \cite{rakhlin2013online}, who also proposed the Optimistic-FTRL. 
%%%%%%%%%%%%%%%%%%%%%%%%%%%%%%%%%%%%%%
%
%%%%%%%%%%%%%%%%%%%%%%%%%%%%%%%%%%%%%%
%Both (Optimistic-)FTRL and (Optimistic-)MD can be unified under one framework in the case of surrogate linearization losses. 
Both DA and MD can be unified via the Unified Mirror Descent (UMD) according to \cite{juditsky2019unifying}. 
%\cite{juditsky2019unifying} unified both DA and MD via the Unified Mirror Descent (UMD). 
However, the level of abstraction and generality of UMD is not enough. 
%\begin{equation*}
%\begin{aligned}
%\quad x_{t}\in\partial\psi^\star\left(x_{t-1}^\psi+\xi_{t-1}\right),\quad x_{t-1}^\psi\in\partial\psi\left(x_{t-1}\right),
%\end{aligned}
%\end{equation*}
%where $\left\{\xi_{t}\right\}_{t\geqslant 1}$ is an arbitrary sequence in $\mathbb{R}^n$. 
%%%%%%%%%%%%%%%%%%%%%%%%%%%%%%%%%%%%%%

%%%%%%%%%%%%%%%%%%%%%%%%%%%%%%%%%%%%%%
For strategies based on surrogate linearization losses, online convex optimization naturally extends to online monotone optimization. 
This idea was first proposed by \cite{gemp2016online} based on earlier research. 
After digging deeper, we found that the novelty of online monotone optimization is to allow the online game to abandon the concept of loss function. 
%Under the unified framework and based on recent research,
Therefore, we argue that the concept of online monotone optimization needs to be rigorously reformulated. 
%%%%%%%%%%%%%%%%%%%%%%%%%%%%%%%%%%%%%%

%%%%%%%%%%%%%%%%%%%%%%%%%%%%%%%%%%%%%%
This paper addresses the online optimization problem from a high-level unified theoretical perspective. 
We not only generalize both Optimistic-DA and Optimistic-MD in normed vector space, but also unify their analysis methods for dynamic regret. 
We focus on optimistic online learning since optimism is at the hub. 
For a non-optimistic version, it suffices to set the estimated term to be null, and for learning with delay, it suffices to modify the estimated term to delete the unobserved loss subgradients \citep{flaspohler2021online}.
Dynamic regret is chosen as the performance metric due to its generality. 
For static regret, it suffices to fix the reference strategies to be constant over time. 
Normed vector space is chosen over $\mathbb{R}^n$ because it shows more essential details. 
The contributions of this paper are as follows. 
%%%%%%%%%%%%%%%%%%%%%%%%%%%%%%%%%%%%%%

%%%%%%%%%%%%%%%%%%%%%%%%%%%%%%%%%%%%%%
\begin{itemize}
%%%%%%%%%%%%%%%%%%%%%%%%%%%%%%%%%%%%%%
\item 
We present a unified analysis method for online optimization in normed vector space using dynamic regret as the the performance metric. 
Our analysis is based on two relaxation strategies, namely $\si$ and $\sii$ (Optimistic-MD), which are obtained through the relaxation of $S$, a two-parameter variant strategy covering Optimistic-DA. 
The analysis process relies on the generalized cosine rule and $\phi$-convex, both of which depend on the generalized Bregman divergence. 
%%%%%%%%%%%%%%%%%%%%%%%%%%%%%%%%%%%%%%
\item 
The regret bounds are the tightest possible. 
Our analysis shows that the last term of all upper bounds is an extra subtraction term, and $\phi$-convex further tightens the upper bounds. 
Instantiations in \cref{sec:instantiations} show that the regret bounds for normalized exponentiated subgradient and greedy/lazy projection are better than the currently known optimal results. 
%%%%%%%%%%%%%%%%%%%%%%%%%%%%%%%%%%%%%%
\item 
All strategies are not only suitable for online convex optimization, but also for online monotone optimization. 
We formalize the online monotone optimization problem, 
and propose the definition of $\regretn{n}$, the generalized version of regret, which allows the absence of losses in online game. 
This is natural and mathematically rigorous. 
%%%%%%%%%%%%%%%%%%%%%%%%%%%%%%%%%%%%%%
\end{itemize}
%%%%%%%%%%%%%%%%%%%%%%%%%%%%%%%%%%%%%%

%%%%%%%%%%%%%%%%%%%%%%%%%%%%%%%%%%%%%%
\section{Preliminaries}
%%%%%%%%%%%%%%%%%%%%%%%%%%%%%%%%%%%%%%

%%%%%%%%%%%%%%%%%%%%%%%%%%%%%%%%%%%%%%
%In this section, we present definitions of generalized Bregman divergence and $\phi$-convex with symmetrical beauty. 
%We emphasize that the generalized cosine rule plays a key role in the derivation of regret upper bounds. 
%%%%%%%%%%%%%%%%%%%%%%%%%%%%%%%%%%%%%%

%%%%%%%%%%%%%%%%%%%%%%%%%%%%%%%%%%%%%%
Let $E$ be a normed vector space over $\mathbb{R}$ and let $E^*$ be the dual space of $E$. 
We denote by $\left\lVert\cdot\right\rVert_E$ the norm of $E$, and denote by $\left\lVert\cdot\right\rVert_{E^*}$ the norm of $E^*$. 
Without causing ambiguity, the subscript of the norm is usually omitted, for example, the space in which the element is located is known. 
We denote by $\left\langle\,\cdot\,, \cdot\,\right\rangle$ the scalar product for the duality $E^*$ and $E$. 
Usually the elements with superscript~$*$ are points in $E^*$.  
%%%%%%%%%%%%%%%%%%%%%%%%%%%%%%%%%%%%%%

%%%%%%%%%%%%%%%%%%%%%%%%%%%%%%%%%%%%%%
We use superscript~$\star$ to denote the Fenchel conjugate of a function. 
Throughout this paper, we introduce $Q_\rho\left(x\right)=\frac{1}{2}x^2+\chi_{\left[-\rho,\rho\right]}\left(x\right)$, where $\forall \rho\in\left(0, +\infty\right]$ and $\forall x\in\mathbb{R}$, and its Fenchel conjugate is $Q_\rho^\star\left(\varkappa\right)=\frac{1}{2}\varkappa^2-\frac{1}{2}\left(\left\lvert\varkappa\right\rvert-\rho\right)_+^2$, where $\forall \varkappa\in\mathbb{R}$, $x_+\coloneqq\max\left\{x, 0\right\}$. 
%%%%%%%%%%%%%%%%%%%%%%%%%%%%%%%%%%%%%%

%%%%%%%%%%%%%%%%%%%%%%%%%%%%%%%%%%%%%%
See \cref{BoCA} for basis of convex analysis. 
%%%%%%%%%%%%%%%%%%%%%%%%%%%%%%%%%%%%%%

%%%%%%%%%%%%%%%%%%%%%%%%%%%%%%%%%%%%%%
Next, we present definitions of generalized Bregman divergence and subdifferential with symmetrical beauty. 
%%%%%%%%%%%%%%%%%%%%%%%%%%%%%%%%%%%%%%

%%%%%%%%%%%%%%%%%%%%%%%%%%%%%%%%%%%%%%
\begin{definition}[Generalized Bregman Divergence]
\label{def:Bregman}
The generalized Bregman divergence w.r.t. a proper function $\varphi$ is defined as 
\begin{equation*}
B_{\varphi}\left(x, y^*\right)\coloneqq\varphi\left(x\right)+\varphi^\star\left(y^*\right)-\left\langle y^*, x\right\rangle,\quad\forall\left(x, y^*\right)\in E\times E^*.
\end{equation*}
\end{definition}
%%%%%%%%%%%%%%%%%%%%%%%%%%%%%%%%%%%%%%

%%%%%%%%%%%%%%%%%%%%%%%%%%%%%%%%%%%%%%
\begin{definition}[Subdifferential]
\label{def:Subdifferential}
The subdifferential of a proper function $\varphi$ at $x$ is 
\begin{equation*}
\begin{aligned}
\partial\varphi\left(x\right)
\coloneqq\left\{x^*\in E^*\left|\,B_{\varphi}\left(x, x^*\right)=0\right.\right\}.
\end{aligned}
\end{equation*}
Any element in $\partial\varphi\left(x\right)$ is a subgradient of $\varphi$ at $x$, denoted by $x^\varphi$. 
\end{definition}
%%%%%%%%%%%%%%%%%%%%%%%%%%%%%%%%%%%%%%

%%%%%%%%%%%%%%%%%%%%%%%%%%%%%%%%%%%%%%
\begin{remark}
This remark illustrates the motivation for the above two definitions. 
The subdifferential of $\varphi$ at $x$ is usually defined as follows \citep[Section~2.4 of][]{zalinescu2002convex}, 
\begin{equation*}
\begin{aligned}
\left\{x^*\in E^*\left|\,\varphi\left(y\right)-\varphi\left(x\right)\geqslant\left\langle x^*,y-x\right\rangle,\,\forall y\in E\right.\right\},
\end{aligned}
\end{equation*}
which is equivalent to $\left\{x^*\in E^*\left|\,\varphi^\star\left(x^*\right)+\varphi\left(x\right)=\left\langle x^*,x\right\rangle\right.\right\}$ according to the definition of Fenchel conjugate. 
Generally, the definition of Bregman divergence is associated with a continuously differentiable function. 
Definition~1 of \cite{pooria2020modular} extends the definition of Bregman divergence by replacing continuously differentiable with directionally differentiable. 
We first consider utilizing subdifferentiable to define Bregman divergence, and then we notice that $\forall y^*\in\partial\varphi\left(y\right)$, 
\[\varphi\left(x\right)-\varphi\left(y\right)-\left\langle y^*, x-y\right\rangle=\varphi\left(x\right)+\varphi^\star\left(y^*\right)-\left\langle y^*, x\right\rangle,\]
which implies that a symmetrical aesthetic expression can be used to extend the definition of Bregman divergence and equivalently describe the definition of subdifferential.  
\end{remark}
%%%%%%%%%%%%%%%%%%%%%%%%%%%%%%%%%%%%%%

%%%%%%%%%%%%%%%%%%%%%%%%%%%%%%%%%%%%%%
Similar to Lemma~3.1 of \cite{chen1993convergence}, we have the following generalized cosine rule, which plays a key role in the derivation of regret upper bounds. 
%%%%%%%%%%%%%%%%%%%%%%%%%%%%%%%%%%%%%%

%%%%%%%%%%%%%%%%%%%%%%%%%%%%%%%%%%%%%%
\begin{lemma}[Generalized Cosine Rule]
\label{Generalized-cosine-law}
If $\left(y, y^*\right)\in\gra\partial\varphi$, then 
\[B_{\varphi}\left(x, y^*\right)+B_{\varphi}\left(y, z^*\right)-B_{\varphi}\left(x, z^*\right)=\left\langle z^*-y^*, x-y\right\rangle.\]
\end{lemma}
%%%%%%%%%%%%%%%%%%%%%%%%%%%%%%%%%%%%%%

%%%%%%%%%%%%%%%%%%%%%%%%%%%%%%%%%%%%%%
\begin{remark}
We call \cref{Generalized-cosine-law} the generalized cosine rule, because if $E$ is a Hilbert space and $\varphi=\frac{1}{2}\left\lVert \cdot\right\rVert^2$, \cref{Generalized-cosine-law} is instantiated as the ordinary cosine rule.  
Indeed, under the above settings, $B_{\frac{1}{2}\left\lVert \cdot\right\rVert^2}\left(x, y^*\right)=\frac{1}{2}\left\lVert  x-y^*\right\rVert^2$, and $y^*=y$ since $\partial\left(\frac{1}{2}\left\lVert \cdot\right\rVert^2\right)$ is the identity map \citep[Proposition~3.6 of][]{chidume2009geometric}. 
Thus, we have 
\[\left\lVert  x-y\right\rVert^2+\left\lVert  y-z^*\right\rVert^2-\left\lVert  x-z^*\right\rVert^2=2\left\langle z^*-y, x-y\right\rangle.\]
\end{remark}
%%%%%%%%%%%%%%%%%%%%%%%%%%%%%%%%%%%%%%

%%%%%%%%%%%%%%%%%%%%%%%%%%%%%%%%%%%%%%
The following lemma plays a pivot role in the equivalent transformation of update rules. 
%%%%%%%%%%%%%%%%%%%%%%%%%%%%%%%%%%%%%%

%%%%%%%%%%%%%%%%%%%%%%%%%%%%%%%%%%%%%%
\begin{lemma}
\label{Symmetry-of-subdifferential}
If $\varphi$ is proper, convex and lower semicontinuous, then
\[x^*\in\partial\varphi\left(x\right)\Longleftrightarrow x\in\partial\varphi^\star\left(x^*\right),\]
and 
\begin{equation*}
\begin{aligned}
\partial\varphi\left(x\right)=\arg\max_{x^*\in E^*}\left\langle x^*,x\right\rangle-\varphi^\star\left(x^*\right),\quad
\partial\varphi^\star\left(x^*\right)=\arg\max_{x\in E}\left\langle x^*,x\right\rangle-\varphi\left(x\right).
\end{aligned}
\end{equation*}
\end{lemma}
%%%%%%%%%%%%%%%%%%%%%%%%%%%%%%%%%%%%%%

%%%%%%%%%%%%%%%%%%%%%%%%%%%%%%%%%%%%%%
Now we define a generalized version of strongly convex --- $\phi$-convex. 
%%%%%%%%%%%%%%%%%%%%%%%%%%%%%%%%%%%%%%

%%%%%%%%%%%%%%%%%%%%%%%%%%%%%%%%%%%%%%
\begin{definition}[$\phi$-Convex]
\label{phi-convex}
A function $\varphi$ is $\phi$-convex if   
\begin{equation*}
B_{\varphi}\left(x, y^*\right)\geqslant\phi\left(\left\lVert x-y\right\rVert\right),\quad\forall x\in E,\quad\forall \left(y, y^*\right)\in\gra\partial\varphi,
\end{equation*}
where $\phi$ is convex and $\phi\geqslant 0$, $\phi\left(0\right)=0$. 
\end{definition} 
%%%%%%%%%%%%%%%%%%%%%%%%%%%%%%%%%%%%%%

%%%%%%%%%%%%%%%%%%%%%%%%%%%%%%%%%%%%%%
\begin{lemma}
\label{affine-convex}
Let $\alpha$ be an affine function. If $\varphi$ is $\phi$-convex, then $\varphi+\alpha$ is $\phi$-convex.
\end{lemma}
%%%%%%%%%%%%%%%%%%%%%%%%%%%%%%%%%%%%%%

%%%%%%%%%%%%%%%%%%%%%%%%%%%%%%%%%%%%%%
\section{Online Convex Optimization}
%%%%%%%%%%%%%%%%%%%%%%%%%%%%%%%%%%%%%%

%%%%%%%%%%%%%%%%%%%%%%%%%%%%%%%%%%%%%%
Let $E$ be a normed vector space over $\mathbb{R}$ and let $C\neq\varnothing$ be a closed convex subset of $E$. 
The online convex optimization problem can be formalized as follows. At round $t$, 
\begin{equation*}
\begin{aligned}
&\text{the player chooses }x_t\in C\text{ according to some algorithm, }\\ 
&\text{the adversary (environment) feeds back a loss function }\varphi_t \text{ with }\dom\partial\varphi_t\supset C. 
\end{aligned}
\end{equation*}
Although we did not specify that $\varphi_t$ is convex, according to \cref{subdifferential-convexity}, $\varphi_t+\chi_C$ is convex and lower semicontinuous, where $\chi_C$ represents the $0 / +\infty$ indicator w.r.t. $\chi_C\left(x\right)=0$ iff $x\in C$. 
We choose the dynamic regret (Equation~\ref{dynamic-regret}) as the performance metric. 
For static regret, it suffices to fix the reference strategies to be constant over time. 
%%%%%%%%%%%%%%%%%%%%%%%%%%%%%%%%%%%%%%

%%%%%%%%%%%%%%%%%%%%%%%%%%%%%%%%%%%%%%
\section{Strategies}
%%%%%%%%%%%%%%%%%%%%%%%%%%%%%%%%%%%%%%

%%%%%%%%%%%%%%%%%%%%%%%%%%%%%%%%%%%%%%
In this section, we derive two relaxation variant forms, namely $\si$ and $\sii$, by relaxing the strategy $S$, a two-parameter variant covering Optimistic-DA. 
The inclusion relationship between strategies in this section is as follows, 
\begin{equation*}
\begin{aligned}
%\left.\begin{array}{ll}
%\si\big|_{\theta_t\equiv 1} \\[8pt]
%\si\big|_{\eta_t\equiv 1}
%\end{array}\right\}\subset
\si\supset S\supset\begin{cases}
S\big|_{\theta_t\equiv 1} = \text{Optimistic-DA,} \\
S\big|_{\eta_t\equiv 1}\subset\sii = \text{Optimistic-MD.}
\end{cases}
\end{aligned}
\end{equation*}
%%%%%%%%%%%%%%%%%%%%%%%%%%%%%%%%%%%%%%

%%%%%%%%%%%%%%%%%%%%%%%%%%%%%%%%%%%%%%
Next, we elaborate on their relationship. 
We start with the definition of $S$. %, which is an optimistic modification of unified mirror descent proposed by \cite{juditsky2019unifying}. 
%%%%%%%%%%%%%%%%%%%%%%%%%%%%%%%%%%%%%%

%%%%%%%%%%%%%%%%%%%%%%%%%%%%%%%%%%%%%%
The strategy $S$ can be formalized as the following two-parameter update rule, 
\begin{equation*}
\begin{aligned}
x_t\in\arg\min_{x\in E}\left\langle \sum_{i=1}^{t-1}\theta_i x_i^* +\theta_t \widehat{x}_t^* ,\, x\right\rangle+\frac{1}{\eta_t}B_{\psi}\left(x, a^\psi\right)
=\partial\psi^\star\left(a^\psi-\eta_t\sum_{i=1}^{t-1}\theta_i x_i^* -\eta_t\theta_t \widehat{x}_t^*\right),
\end{aligned}
\end{equation*}
where $\widehat{x}_t^*\in E^*$ is the estimated linear loss function corresponding to $x_t^*$, 
$\eta_t>0$ is the cumulative parameter since it acts on the cumulative quantity $\sum_{i=1}^{t-1}\theta_i x_i^*$, 
$\theta_t>0$ is the instantaneous parameter since it acts on the instantaneous quantity $x_t^*$, 
$\psi\colon\! C\rightarrow\mathbb{R}$ is $
\phi$-convex, %and lower semicontinuous, 
$\left(a,a^\psi\right)\in\gra\partial\psi$, 
and the equality follows from \cref{Symmetry-of-subdifferential}.  
%We emphasize that $\partial\psi^\star$ is a monotone multivalued operator. 
The equivalent two-step iterative form can be rearranged as follows, 
\begin{equation*}
\begin{aligned}
\widetilde{x}_{t}^\psi=a^\psi-\eta_t\sum_{i=1}^{t-1}\theta_i x_i^*, \quad
x_{t}\in\partial\psi^\star\left(\widetilde{x}_{t}^\psi-\eta_t\theta_{t} \widehat{x}_{t}^*\right),
\end{aligned}
\end{equation*}
where $\widetilde{x}_{t}^\psi$ is the intermediate variable. The corresponding function form of $S$ can be written as 
$\left(x_{t}, \widetilde{x}_{t}^\psi\right)=S\left(\widehat{x}_{t}^*, x_{t-1}^*; \eta_t, \theta_{t}\right)$, $\widetilde{x}_{1}^\psi=a^\psi$. 
%%%%%%%%%%%%%%%%%%%%%%%%%%%%%%%%%%%%%%

%%%%%%%%%%%%%%%%%%%%%%%%%%%%%%%%%%%%%%
Most of the literature explores strategies in the form of $\arg\min$, this paper elaborates from the perspective of multivalued maps, which makes the analysis simpler. 
%%%%%%%%%%%%%%%%%%%%%%%%%%%%%%%%%%%%%%

%%%%%%%%%%%%%%%%%%%%%%%%%%%%%%%%%%%%%%
\begin{remark}
If $\theta_t\equiv\theta$, then the strategy $S$ degenerates into Optimistic-DA, abbreviated as $S\big|_{\theta_t\equiv 1}$. 
Indeed, it suffices to consider the case $\theta_t\equiv 1$, that is,  
\begin{equation*}
\begin{aligned}
x_t\in\arg\min_{x\in E} \left\langle\sum_{i=1}^{t-1}x_i^* + \widehat{x}_t^*, x\right\rangle+\frac{1}{\eta_t}B_{\psi}\left(x, a^\psi\right). 
\end{aligned}
\end{equation*}
\end{remark}
%%%%%%%%%%%%%%%%%%%%%%%%%%%%%%%%%%%%%%

%%%%%%%%%%%%%%%%%%%%%%%%%%%%%%%%%%%%%%
The type-I relaxation variant form of $S$ (named as $\si$) is formalized as 
\begin{equation*}
\begin{aligned}
\widecheck{x}_{t}^\psi&=a^\psi-\eta_t\sum_{i=1}^{t-1}\theta_i x_i^*, 
\quad\widetilde{x}_{t}\in\partial\psi^\star\left(\widecheck{x}_{t}^\psi\right), 
\quad\widetilde{x}_{t}^\psi\in\partial\psi\left(\widetilde{x}_{t}\right), \\
x_{t}&\in\partial\psi^\star\left(\widetilde{x}_{t}^\psi-\eta_t\theta_{t} \widehat{x}_{t}^*\right).
\end{aligned}
\end{equation*}
The function form of $\si$ can be written as 
$\left(x_{t}, \widecheck{x}_{t}^\psi, \widetilde{x}_{t}^\psi\right)=\si\left(\widehat{x}_{t}^*, x_{t-1}^*; \eta_t, \theta_{t}\right)$, $\widetilde{x}_{1}^\psi=a^\psi$. 
%%%%%%%%%%%%%%%%%%%%%%%%%%%%%%%%%%%%%%

%%%%%%%%%%%%%%%%%%%%%%%%%%%%%%%%%%%%%%
\begin{remark}
$\si$ allows $\widetilde{x}_{t}^\psi\neq \widecheck{x}_{t}^\psi$, which makes $S$ a special case of $\si$. 
Indeed, $\widetilde{x}_{t}^\psi$ and $\widecheck{x}_{t}^\psi$ are both elements of $\partial\psi\left(\widetilde{x}_{t}\right)$ according to \cref{Symmetry-of-subdifferential}. 
\end{remark}
%%%%%%%%%%%%%%%%%%%%%%%%%%%%%%%%%%%%%%

%%%%%%%%%%%%%%%%%%%%%%%%%%%%%%%%%%%%%%
If $\eta_t\equiv\eta$, then the strategy $S$ degenerates into the following form  (without loss of generality, one can set $\eta_t\equiv 1$, and denote it as $S\big|_{\eta_t\equiv 1}$), 
\begin{equation*}
\begin{aligned}
\widetilde{x}_{t}^\psi=\widetilde{x}_{t-1}^\psi-\theta_{t-1} x_{t-1}^*, \quad 
\quad x_{t}\in\partial\psi^\star\left(\widetilde{x}_{t}^\psi-\theta_{t} \widehat{x}_{t}^*\right).
\end{aligned}
\end{equation*}
The type-II relaxation variant form of $S$ (named as $\sii$) can be formalized as 
\begin{equation*}
\begin{aligned}
\widecheck{x}_{t}^\psi&=\widetilde{x}_{t-1}^\psi-\theta_{t-1} x_{t-1}^*, 
\quad\widetilde{x}_{t}\in\partial\psi^\star\left(\widecheck{x}_{t}^\psi\right), 
\quad\widetilde{x}_{t}^\psi\in\partial\psi\left(\widetilde{x}_{t}\right), \quad \widetilde{x}_{1}^\psi=a^\psi, \\
x_{t}&\in\partial\psi^\star\left(\widetilde{x}_{t}^\psi-\theta_{t} \widehat{x}_{t}^*\right).
\end{aligned}
\end{equation*}
The function form of $\sii$ can be written as 
$\left(x_{t}, \widecheck{x}_{t}^\psi, \widetilde{x}_{t}^\psi\right)=\sii\left(\widehat{x}_{t}^*, x_{t-1}^*; \theta_{t}\right)$, $\widetilde{x}_{1}^\psi=a^\psi$. 
%%%%%%%%%%%%%%%%%%%%%%%%%%%%%%%%%%%%%%

%%%%%%%%%%%%%%%%%%%%%%%%%%%%%%%%%%%%%%
\begin{remark}
$\sii$ allows $\widetilde{x}_{t}^\psi\neq \widecheck{x}_{t}^\psi$, which makes $S\big|_{\eta_t\equiv 1}$ a special case of $\sii$. 
In fact, $\sii$ is Optimistic-MD in normed vector space. 
\end{remark}
%%%%%%%%%%%%%%%%%%%%%%%%%%%%%%%%%%%%%%

%%%%%%%%%%%%%%%%%%%%%%%%%%%%%%%%%%%%%%
\section{Regret Analysis}
%%%%%%%%%%%%%%%%%%%%%%%%%%%%%%%%%%%%%%

%%%%%%%%%%%%%%%%%%%%%%%%%%%%%%%%%%%%%%
In this section, we use a unified analysis method to prove the regret upper bounds for $S$, $\si$ and $\sii$. 
We start with the generalized strategy $\si$. 
%%%%%%%%%%%%%%%%%%%%%%%%%%%%%%%%%%%%%%

%%%%%%%%%%%%%%%%%%%%%%%%%%%%%%%%%%%%%%
\begin{theorem}[Dynamic Regret for $\si$]
\label{S-I}
If $\partial\psi\left(C\right)=E^*$ or $C$ is compact, then $\forall a\in C$,  $\si$ enjoys the following dynamic regret upper bound, 
\begin{equation*}
\begin{aligned}
\regret\left(z_1,z_2,\cdots,z_T\right)
\leqslant\ &\sum_{t=1}^{T}\frac{1}{\eta_t\theta_t}\left[B_{\psi}\left(z_t, \widecheck{x}_t^\psi\right)-B_{\psi}\left(z_t, a^\psi+\frac{\eta_t}{\eta_{t+1}}\left(\widecheck{x}_{t+1}^\psi-a^\psi\right)\right)\right] \\
&+\sum_{t=1}^{T}\frac{1}{\eta_t\theta_t}\left[B_{\psi}\left(X_{t+1}, \widetilde{x}_t^\psi\right)-B_{\psi}\left(X_{t+1}, \widecheck{x}_t^\psi\right)\right] \\
&+\sum_{t=1}^{T}\frac{1}{\eta_t\theta_t}\phi^\star\left(\eta_t\theta_t\left\lVert x_t^*-\widehat{x}_t^*\right\rVert\right)
-\sum_{t=1}^{T}\frac{1}{\eta_t\theta_t}B_{\psi}\left(x_{t}, \widetilde{x}_{t}^\psi\right), \quad\forall z_t\in C,
\end{aligned}
\end{equation*}
where $\left(X_{t+1}, \frac{1}{\eta_{t+1}}\left(\widetilde{x}_{t+1}^{\psi}-a^{\psi}\right)\right)\in\gra\partial\frac{1}{\eta_t}\left(\psi-a^\psi\right)$. 
\end{theorem}
%%%%%%%%%%%%%%%%%%%%%%%%%%%%%%%%%%%%%%

%%%%%%%%%%%%%%%%%%%%%%%%%%%%%%%%%%%%%%
\begin{proof}
$\partial\psi\left(C\right)=E^*$ or the compactness of $C$ are sufficient conditions to guarantee that $x_t\in C$. 
See \cref{pf:S-I} for general analysis. 
The following proof focuses on the derivation of the dynamic regret upper bound for $\si$. 

$\si$ can be rearranged as follows, 
\begin{equation*}
x_t^{\psi}+\eta_t\theta_t \widehat{x}_t^*=\widetilde{x}_t^{\psi}, 
\quad\widecheck{x}_t^\psi=a^\psi-\eta_t\sum_{i=1}^{t-1}\theta_i x_i^*=a^{\psi}+\eta_t X_t^{\widetilde{\psi}}, 
\end{equation*}
where $X_t^{\widetilde{\psi}}$ is an auxiliary variable, $\widetilde{\psi}\colon\!C\rightarrow\mathbb{R}$ is a temporarily unknown convex function. 

Note that 
\begin{equation}
\label{S-I-proof-1}
\varphi_t\left(x_t\right)-\varphi_t\left(z_t\right)
\leqslant\frac{1}{\theta_t}\left\langle \theta_t x_t^*,x_t-z_t\right\rangle, \quad x_t^*\in\partial\varphi_t\left(x_t\right),
\end{equation}
and 
\begin{align}
&\left\langle \theta_t x_t^*,x_t-z_t\right\rangle
=\left\langle \theta_t x_t^*, X_{t+1}-z_t\right\rangle+\left\langle \theta_t \widehat{x}_{t}^*,x_t-X_{t+1}\right\rangle+\left\langle  \theta_t x_t^*-\theta_t \widehat{x}_t^*,x_t-X_{t+1}\right\rangle \notag\\
=\ &-\left\langle X_t^{\widetilde{\psi}}-X_{t+1}^{\widetilde{\psi}},z_t-X_{t+1}\right\rangle-\frac{1}{\eta_t}\left\langle \widetilde{x}_{t}^\psi-x_{t}^\psi,X_{t+1}-x_t\right\rangle+\left\langle \theta_t x_t^*-\theta_t \widehat{x}_t^*,x_t-X_{t+1}\right\rangle \notag\\
=\ &B_{\widetilde{\psi}}\left(z_t, X_t^{\widetilde{\psi}}\right)-B_{\widetilde{\psi}}\left(z_t, X_{t+1}^{\widetilde{\psi}}\right)-B_{\widetilde{\psi}}\left(X_{t+1}, X_t^{\widetilde{\psi}}\right)+\frac{1}{\eta_t}B_{\psi}\left(X_{t+1}, \widetilde{x}_t^\psi\right)-\frac{1}{\eta_t}B_{\psi}\left(x_t, \widetilde{x}_t^\psi\right) \notag\\
&-\frac{1}{\eta_t}B_{\psi}\left(X_{t+1}, x_t^\psi\right)+\frac{1}{\eta_t}\left\langle \eta_t\theta_t x_t^*-\eta_t\theta_t \widehat{x}_t^*,x_t-X_{t+1}\right\rangle \notag\\
\leqslant\ &B_{\widetilde{\psi}}\left(z_t, X_t^{\widetilde{\psi}}\right)-B_{\widetilde{\psi}}\left(z_t, X_{t+1}^{\widetilde{\psi}}\right)-B_{\widetilde{\psi}}\left(X_{t+1}, X_t^{\widetilde{\psi}}\right)+\frac{1}{\eta_t}B_{\psi}\left(X_{t+1}, \widetilde{x}_t^\psi\right)-\frac{1}{\eta_t}B_{\psi}\left(x_t, \widetilde{x}_t^\psi\right) \notag\\
&+\frac{1}{\eta_t}\phi^\star\left(\eta_t\theta_t\left\lVert x_t^*-\widehat{x}_t^*\right\rVert\right), \label{S-I-proof-2}
\end{align}
where $X_{t+1}\in\partial\widetilde{\psi}^\star\left(X_{t+1}^{\widetilde{\psi}}\right)$, the last ``$=$'' uses the generalized cosine rule (\cref{Generalized-cosine-law}), and ``$\leqslant$'' uses the $\phi$-convexity of $\psi$ and the following inequality, 
\begin{equation*}
\begin{aligned}
\left\langle \eta_t\theta_t x_t^*-\eta_t\theta_t \widehat{x}_t^*,x_t-X_{t+1}\right\rangle
&\leqslant\eta_t\theta_t\left\lVert x_t^*-\widehat{x}_t^*\right\rVert\left\lVert x_t-X_{t+1}\right\rVert \\
&\leqslant\phi\left(\left\lVert x_t-X_{t+1}\right\rVert\right)+\phi^\star\left(\eta_t\theta_t\left\lVert x_t^*-\widehat{x}_t^*\right\rVert\right).
\end{aligned}
\end{equation*}
If $\widetilde{\psi}=\frac{1}{\eta_t}\left(\psi-a^\psi\right)$, then $\widetilde{\psi}^\star\left(y^*\right)=\frac{1}{\eta_t}\psi^\star\left(a^\psi+\eta_t y^*\right)$, according to $X_t^{\widetilde{\psi}}=\frac{1}{\eta_t}\left(\widecheck{x}_t^{\psi}-a^{\psi}\right)$ and $X_{t+1}^{\widetilde{\psi}}=\frac{1}{\eta_{t+1}}\left(\widecheck{x}_{t+1}^{\psi}-a^{\psi}\right)$, we have 
\begin{equation}
\label{S-I-proof-3}
\begin{aligned}
B_{\widetilde{\psi}}\left(z_t, X_t^{\widetilde{\psi}}\right)
&=\frac{1}{\eta_t}B_{\psi}\left(z_t, \widecheck{x}_t^\psi\right), \\
B_{\widetilde{\psi}}\left(z_t, X_{t+1}^{\widetilde{\psi}}\right)
&=\frac{1}{\eta_t}B_{\psi}\left(z_t, a^\psi+\frac{\eta_t}{\eta_{t+1}}\left(\widecheck{x}_{t+1}^\psi-a^\psi\right)\right), \\
B_{\widetilde{\psi}}\left(X_{t+1}, X_t^{\widetilde{\psi}}\right)
%&=\widetilde{\psi}\left(X_{t+1}\right)+\widetilde{\psi}^\star\left(X_t^{\widetilde{\psi}}\right)-\left\langle X_t^{\widetilde{\psi}}, X_{t+1}\right\rangle \\
%&=\left(\widetilde{\psi}+\frac{x_1^\psi}{\eta_{t}}\right)X_{t+1}+\widetilde{\psi}^\star\left(\frac{\widetilde{x}_t^\psi-x_1^\psi}{\eta_{t}}\right)-\frac{1}{\eta_{t}}\left\langle \widetilde{x}_t^\psi, X_{t+1}\right\rangle
&=\frac{1}{\eta_t}B_{\psi}\left(X_{t+1}, \widecheck{x}_t^\psi\right).  
\end{aligned}
\end{equation}
To complete the proof, it suffices to combine \cref{S-I-proof-1,S-I-proof-2,S-I-proof-3}. 
\end{proof}
%%%%%%%%%%%%%%%%%%%%%%%%%%%%%%%%%%%%%%

%%%%%%%%%%%%%%%%%%%%%%%%%%%%%%%%%%%%%%
\begin{remark}
The proof of \cref{S-I} provides a general analysis method for dynamic regret. 
The decomposition of the instantaneous dynamic regret (the first line of Equation~\ref{S-I-proof-2}) follows from Appendix~A.1 of  \cite{zhao2020dynamic}. 
In fact, \cite{zhao2020dynamic} deduced the dynamic regret for online extra-gradient descent, which is a special case of $\sii$. 
Later in this section (\cref{S-II}), we show the dynamic regret for $\sii$. 
\end{remark}
%%%%%%%%%%%%%%%%%%%%%%%%%%%%%%%%%%%%%%

%%%%%%%%%%%%%%%%%%%%%%%%%%%%%%%%%%%%%%
Next, we analyze each term of the dynamic regret upper bound for $\si$. 
%%%%%%%%%%%%%%%%%%%%%%%%%%%%%%%%%%%%%%

%%%%%%%%%%%%%%%%%%%%%%%%%%%%%%%%%%%%%%
In order to reduce the item 
\begin{equation*}
\begin{aligned}
\sum_{t=1}^{T}\frac{1}{\eta_t\theta_t}\left[B_{\psi}\left(z_t, \widecheck{x}_t^\psi\right)-B_{\psi}\left(z_t, a^\psi+\frac{\eta_t}{\eta_{t+1}}\left(\widecheck{x}_{t+1}^\psi-a^\psi\right)\right)\right], 
\end{aligned}
\end{equation*}
one feasible way is to set $\theta_t\equiv \theta$ and $z_t\equiv z$, the other is to set $\eta_t\equiv \eta$ and $z_t\equiv z$. 
Indeed, we have the following lemma. 
%%%%%%%%%%%%%%%%%%%%%%%%%%%%%%%%%%%%%%

%%%%%%%%%%%%%%%%%%%%%%%%%%%%%%%%%%%%%%
\begin{lemma}
\label{static-reduce}
If $\eta_t\geqslant\eta_{t+1}$, then 
\begin{equation*}
\begin{aligned}
\sum_{t=1}^{T}\frac{1}{\eta_t}\left[B_{\psi}\left(z, \widecheck{x}_t^\psi\right)-B_{\psi}\left(z, a^\psi+\frac{\eta_t}{\eta_{t+1}}\left(\widecheck{x}_{t+1}^\psi-a^\psi\right)\right)\right]
\leqslant
\frac{1}{\eta_{T+1}}B_{\psi}\left(z, a^\psi\right). 
\end{aligned}
\end{equation*}
If $\theta_{t-1}\leqslant\theta_{t}$, then 
\begin{equation*}
\begin{aligned}
\sum_{t=1}^{T}\frac{1}{\theta_t}\left[B_{\psi}\left(z_t, \widecheck{x}_{t}^\psi\right)-B_{\psi}\left(z_t, \widecheck{x}_{t+1}^\psi\right)\right]\leqslant
\frac{1}{\theta_{1}}B_{\psi}\left(z, a^\psi\right)
+\sum_{t=2}^{T}\frac{1}{\theta_t}\left\lVert \partial\psi\left(z_t\right)-\widecheck{x}_{t}^\psi\right\rVert\left\lVert z_t-z_{t-1}\right\rVert.
\end{aligned}
\end{equation*}
\end{lemma}
%%%%%%%%%%%%%%%%%%%%%%%%%%%%%%%%%%%%%%

%%%%%%%%%%%%%%%%%%%%%%%%%%%%%%%%%%%%%%
Combining \cref{S-I} and \cref{static-reduce}, we directly obtain the static regret for $\si$.
%%%%%%%%%%%%%%%%%%%%%%%%%%%%%%%%%%%%%%

%%%%%%%%%%%%%%%%%%%%%%%%%%%%%%%%%%%%%%
\begin{corollary}[Static Regret for $\si$]
\label{S-I-static}
If $\partial\psi\left(C\right)=E^*$ or $C$ is compact, then $\forall a\in C$, 
under the assumption of $\eta_t\geqslant\eta_{t+1}$, $\si\big|_{\theta_t\equiv 1}$ enjoys the following static regret upper bound, 
\begin{equation*}
\begin{aligned}
\regret\left(z,z,\cdots,z\right)
\leqslant\ &\frac{1}{\eta_{T+1}}B_{\psi}\left(z, a^\psi\right)+\sum_{t=1}^{T}\frac{1}{\eta_t}\left[B_{\psi}\left(X_{t+1}, \widetilde{x}_t^\psi\right)-B_{\psi}\left(X_{t+1}, \widecheck{x}_t^\psi\right)\right] \\
&+\sum_{t=1}^{T}\frac{1}{\eta_t}\phi^\star\left(\eta_t\left\lVert x_t^*-\widehat{x}_t^*\right\rVert\right)
-\sum_{t=1}^{T}\frac{1}{\eta_t}B_{\psi}\left(x_{t}, \widetilde{x}_{t}^\psi\right),\quad\forall z\in C,
\end{aligned}
\end{equation*}
where $\left(X_{t+1}, \frac{1}{\eta_{t+1}}\left(\widetilde{x}_{t+1}^{\psi}-a^{\psi}\right)\right)\in\gra\partial\frac{1}{\eta_t}\left(\psi-a^\psi\right)$, 
and under the assumption of $\theta_{t-1}\leqslant\theta_t$, $\si\big|_{\eta_t\equiv 1}$ enjoys the following static regret upper bound, 
\begin{equation*}
\begin{aligned}
\regret\left(z,z,\cdots,z\right)
\leqslant\ &\frac{1}{\theta_{1}}B_{\psi}\left(z, a^\psi\right)+\sum_{t=1}^{T}\frac{1}{\theta_t}\left[B_{\psi}\left(X_{t+1}, \widetilde{x}_t^\psi\right)-B_{\psi}\left(X_{t+1}, \widecheck{x}_t^\psi\right)\right] \\
&+\sum_{t=1}^{T}\frac{1}{\theta_t}\phi^\star\left(\theta_t\left\lVert x_t^*-\widehat{x}_t^*\right\rVert\right)
-\sum_{t=1}^{T}\frac{1}{\theta_t}B_{\psi}\left(x_{t}, \widetilde{x}_{t}^\psi\right),\quad\forall z\in C,
\end{aligned}
\end{equation*}
where $\left(X_{t+1}, \widetilde{x}_{t+1}^{\psi}-a^{\psi}\right)\in\gra\partial\left(\psi-a^\psi\right)$. 
\end{corollary}
%%%%%%%%%%%%%%%%%%%%%%%%%%%%%%%%%%%%%%

%%%%%%%%%%%%%%%%%%%%%%%%%%%%%%%%%%%%%%
There are two ways to drop the term 
$B_{\psi}\left(X_{t+1}, \widetilde{x}_t^\psi\right)-B_{\psi}\left(X_{t+1}, \widecheck{x}_t^\psi\right)$. 
One is to set $\widetilde{x}_{t}^\psi= \widecheck{x}_{t}^\psi$, which forces $\si$ back to $S$. 
The regret for $S$ is formalized into the following two corollaries. 
The other is to set $\psi$ be the squared norm on Hilbert space. 
See \cref{sec:SN} for analysis details.  
%%%%%%%%%%%%%%%%%%%%%%%%%%%%%%%%%%%%%%

%%%%%%%%%%%%%%%%%%%%%%%%%%%%%%%%%%%%%%
\begin{corollary}[Dynamic Regret for $S$]
\label{S}
If $\partial\psi\left(C\right)=E^*$ or $C$ is compact, then $\forall a\in C$, $S$ enjoys the following dynamic regret upper bound, 
\begin{equation*}
\begin{aligned}
\regret\left(z_1,z_2,\cdots,z_T\right)
\leqslant\ &\sum_{t=1}^{T}\frac{1}{\eta_t\theta_t}\left[B_{\psi}\left(z_t, \widetilde{x}_t^\psi\right)-B_{\psi}\left(z_t, a^\psi+\frac{\eta_t}{\eta_{t+1}}\left(\widetilde{x}_{t+1}^\psi-a^\psi\right)\right)\right] \\
&+\sum_{t=1}^{T}\frac{1}{\eta_t\theta_t}\phi^\star\left(\eta_t\theta_t\left\lVert x_t^*-\widehat{x}_t^*\right\rVert\right)
-\sum_{t=1}^{T}\frac{1}{\eta_t\theta_t}B_{\psi}\left(x_{t}, \widetilde{x}_{t}^\psi\right), \quad\forall z_t\in C. 
\end{aligned}
\end{equation*}
\end{corollary}
%%%%%%%%%%%%%%%%%%%%%%%%%%%%%%%%%%%%%%

%%%%%%%%%%%%%%%%%%%%%%%%%%%%%%%%%%%%%%
\begin{corollary}[Static Regret for $S$]
\label{S-static}
If $\partial\psi\left(C\right)=E^*$ or $C$ is compact, then $\forall a\in C$, 
under the assumption of $\eta_t\geqslant\eta_{t+1}$,  $S\big|_{\theta_t\equiv 1}$ (which is Optimistic-DA) enjoys the following static regret upper bound,  
\begin{equation*}
\begin{aligned}
\regret\left(z,z,\cdots,z\right)
\leqslant\frac{1}{\eta_{T+1}}B_{\psi}\left(z, a^\psi\right)+\sum_{t=1}^{T}\frac{1}{\eta_t}\phi^\star\left(\eta_t\left\lVert x_t^*-\widehat{x}_t^*\right\rVert\right)
-\sum_{t=1}^{T}\frac{1}{\eta_t}B_{\psi}\left(x_{t}, \widetilde{x}_{t}^\psi\right),\ \forall z\in C,
\end{aligned}
\end{equation*}
and under the assumption of $\theta_{t-1}\leqslant \theta_{t}$, $S\big|_{\eta_t\equiv 1}$ enjoys the following static regret upper bound,  
\begin{equation*}
\begin{aligned}
\regret\left(z,z,\cdots,z\right)
\leqslant\frac{1}{\theta_{1}}B_{\psi}\left(z, a^\psi\right)+\sum_{t=1}^{T}\frac{1}{\theta_t}\phi^\star\left(\theta_t\left\lVert x_t^*-\widehat{x}_t^*\right\rVert\right)
-\sum_{t=1}^{T}\frac{1}{\theta_t}B_{\psi}\left(x_{t}, \widetilde{x}_{t}^\psi\right),\ \ \forall z\in C.
\end{aligned}
\end{equation*}
\end{corollary}
%%%%%%%%%%%%%%%%%%%%%%%%%%%%%%%%%%%%%%

%%%%%%%%%%%%%%%%%%%%%%%%%%%%%%%%%%%%%%
Note that the term $\phi^\star\left(\eta_t\theta_t\left\lVert x_t^*-\widehat{x}_t^*\right\rVert\right)$ contains the function $\phi$. 
Compared with the strongly-convex, $\phi$-convex allows the regret upper bounds to be more finely controlled. 
For example, the upper bounds of using $Q_\rho$-convex are tighter than that of using $1$-strongly-convex, since $Q_\rho^\star$ has an extra subtraction term, where $\rho=\sup_{x,y\in C}\left\lVert x-y\right\rVert$.
%%%%%%%%%%%%%%%%%%%%%%%%%%%%%%%%%%%%%%

%%%%%%%%%%%%%%%%%%%%%%%%%%%%%%%%%%%%%%
For the term $B_{\psi}\left(x_{t}, \widetilde{x}_{t}^\psi\right)$, one can relax regret bounds via the $\phi$-convexity of $\psi$, that is, 
\begin{equation*}
\begin{aligned}
B_{\psi}\left(x_{t}, \widetilde{x}_{t}^\psi\right)
\geqslant\phi\left(\left\lVert x_{t}-\widetilde{x}_{t}\right\rVert\right). 
\end{aligned}
\end{equation*}
%%%%%%%%%%%%%%%%%%%%%%%%%%%%%%%%%%%%%%

%%%%%%%%%%%%%%%%%%%%%%%%%%%%%%%%%%%%%%
Now we show the dynamic regret for $\sii$, which is the general form of Appendix~A.1 of  \cite{zhao2020dynamic}. 
%%%%%%%%%%%%%%%%%%%%%%%%%%%%%%%%%%%%%%

%%%%%%%%%%%%%%%%%%%%%%%%%%%%%%%%%%%%%%
\begin{proposition}[Dynamic Regret for $\sii$]
\label{S-II}
If $\partial\psi\left(C\right)=E^*$ or $C$ is compact, then $\forall a\in C$,  $\sii$ (which is Optimistic-MD) enjoys the following dynamic regret upper bound, 
\begin{equation*}
\begin{aligned}
\regret\left(z_1,z_2,\cdots,z_T\right)
\leqslant\ &\sum_{t=1}^{T}\frac{1}{\theta_t}\left[B_{\psi}\left(z_t, \widetilde{x}_{t}^\psi\right)-B_{\psi}\left(z_t, \widecheck{x}_{t+1}^\psi\right)\right] \\
&+\sum_{t=1}^{T}\frac{1}{\theta_t}\phi^\star\left(\theta_t\left\lVert x_t^*-\widehat{x}_t^*\right\rVert\right)
-\sum_{t=1}^{T}\frac{1}{\theta_t}B_{\psi}\left(x_{t}, \widetilde{x}_{t}^\psi\right), \quad\forall z_t\in C.
\end{aligned}
\end{equation*}
\end{proposition}
%%%%%%%%%%%%%%%%%%%%%%%%%%%%%%%%%%%%%%

%%%%%%%%%%%%%%%%%%%%%%%%%%%%%%%%%%%%%%
\begin{proof}
This proof is similar to the proof of \cref{S-I}. 
Note that $\sii$ can be rearranged as follows,
\begin{equation*}
\begin{aligned}
x_{t}^\psi+\theta_{t} \widehat{x}_{t}^*=\widetilde{x}_{t}^\psi,\quad \widecheck{x}_{t}^\psi=\widetilde{x}_{t-1}^\psi-\theta_{t-1} x_{t-1}^*, 
\end{aligned}
\end{equation*}
and 
\begin{equation}
\label{S-II-proof-1}
\varphi_t\left(x_t\right)-\varphi_t\left(z_t\right)
\leqslant\frac{1}{\theta_t}\left\langle \theta_t x_t^*,x_t-z_t\right\rangle, \quad x_t^*\in\partial\varphi_t\left(x_t\right),
\end{equation}
where 
\begin{align}
&\left\langle \theta_t x_t^*,x_t-z_t\right\rangle
=\left\langle \theta_t x_t^*, \widecheck{x}_{t+1}-z_t\right\rangle+\left\langle \theta_t \widehat{x}_{t}^*,x_t-\widecheck{x}_{t+1}\right\rangle+\left\langle  \theta_t x_t^*-\theta_t \widehat{x}_t^*,x_t-\widecheck{x}_{t+1}\right\rangle \notag\\
=\ &-\left\langle \widetilde{x}_{t}^\psi-\widecheck{x}_{t+1}^\psi,z_t-\widecheck{x}_{t+1}\right\rangle-\left\langle \widetilde{x}_{t}^\psi-x_{t}^\psi,\widecheck{x}_{t+1}-x_t\right\rangle+\left\langle \theta_t x_t^*-\theta_t \widehat{x}_t^*,x_t-\widecheck{x}_{t+1}\right\rangle \notag\\
=\ &B_{\psi}\left(z_t, \widetilde{x}_{t}^\psi\right)-B_{\psi}\left(z_t, \widecheck{x}_{t+1}^\psi\right)
%-B_{\psi}\left(\widecheck{x}_{t+1}, \widetilde{x}_{t}^\psi\right)+B_{\psi}\left(\widecheck{x}_{t+1}, \widetilde{x}_t^\psi\right)
-B_{\psi}\left(x_t, \widetilde{x}_t^\psi\right)-B_{\psi}\left(\widecheck{x}_{t+1}, x_t^\psi\right)+\left\langle \theta_t x_t^*-\theta_t \widehat{x}_t^*,x_t-\widecheck{x}_{t+1}\right\rangle \notag\\
\leqslant\ &B_{\psi}\left(z_t, \widetilde{x}_{t}^\psi\right)-B_{\psi}\left(z_t, \widecheck{x}_{t+1}^\psi\right)-B_{\psi}\left(x_t, \widetilde{x}_t^\psi\right)+\phi^\star\left(\theta_t\left\lVert x_t^*-\widehat{x}_t^*\right\rVert\right). \label{S-II-proof-2}
\end{align}
To complete the proof, it suffices to combine \cref{S-II-proof-1,S-II-proof-2}. 
\end{proof}
%%%%%%%%%%%%%%%%%%%%%%%%%%%%%%%%%%%%%%

%%%%%%%%%%%%%%%%%%%%%%%%%%%%%%%%%%%%%%
Note that 
\begin{equation*}
\begin{aligned}
\sum_{t=1}^{T}\frac{1}{\theta_t}\left[B_{\psi}\left(z_t, \widetilde{x}_{t}^\psi\right)-B_{\psi}\left(z_t, \widecheck{x}_{t+1}^\psi\right)\right]
=\ &\sum_{t=1}^{T}\frac{1}{\theta_t}\left[B_{\psi}\left(z_t, \widetilde{x}_{t}^\psi\right)-B_{\psi}\left(z_t, \widetilde{x}_{t+1}^\psi\right)\right] \\
&+\sum_{t=1}^{T}\frac{1}{\theta_t}\left[B_{\psi}\left(z_t, \widetilde{x}_{t+1}^\psi\right)-B_{\psi}\left(z_t, \widecheck{x}_{t+1}^\psi\right)\right].
\end{aligned}
\end{equation*}
This decomposition makes the bound of \cref{S-II} extremely similar to that of \cref{S-I}. 
The method of reducing the term $\sum_{t=1}^{T}\frac{1}{\theta_t}\left[B_{\psi}\left(z_t, \widetilde{x}_{t}^\psi\right)-B_{\psi}\left(z_t, \widetilde{x}_{t+1}^\psi\right)\right]$ is covered by \cref{static-reduce}. 
For the term $B_{\psi}\left(z_t, \widetilde{x}_{t+1}^\psi\right)-B_{\psi}\left(z_t, \widecheck{x}_{t+1}^\psi\right)$, there are two feasible ways to drop it. 
One is to set $\widetilde{x}_{t+1}^\psi= \widecheck{x}_{t+1}^\psi$, which forces $\sii$ back to $S\big|_{\eta_t\equiv 1}$. 
The other is to set $\psi$ be the squared norm on Hilbert space. 
See \cref{sec:SN} for analysis details.  
%%%%%%%%%%%%%%%%%%%%%%%%%%%%%%%%%%%%%%

%%%%%%%%%%%%%%%%%%%%%%%%%%%%%%%%%%%%%%
Next, we investigate the effect of introducing auxiliary strategies in addition to primary strategies on regret bounds. 
This is a general extension of Appendix~B of \cite{flaspohler2021online}. 
%%%%%%%%%%%%%%%%%%%%%%%%%%%%%%%%%%%%%%

%%%%%%%%%%%%%%%%%%%%%%%%%%%%%%%%%%%%%%
Note that $\left\{x_t\right\}_{t\geqslant 1}$ is determined by the estimated sequence $\left\{\widehat{x}_t^*\right\}_{t\geqslant 1}$, which is arbitrary. 
Without changing the primary strategy sequence $\left\{x_t\right\}_{t\geqslant 1}$, we introduce an auxiliary strategy sequence $\left\{y_t\right\}_{t\geqslant 1}$, which is determined by $\left\{\widehat{y}_t^*\right\}_{t\geqslant 1}$. 
The instantaneous dynamic regret can be decomposed as follows, 
\begin{equation*}
\begin{aligned}
\varphi_t\left(x_t\right)-\varphi_t\left(z_t\right)
\leqslant\left\langle x_t^*,x_t-z_t\right\rangle
=\underbrace{\left\langle x_t^*,x_t-y_t\right\rangle}_{\text{drift}}+\underbrace{\left\langle x_t^*,y_t-z_t\right\rangle}_{\text{auxiliary}}, \quad x_t^*\in\partial\varphi_t\left(x_t\right).
\end{aligned}
\end{equation*}
The regret bound for the auxiliary term is simply replacing $x_{t}$, $\widecheck{x}_{t}^\psi$ and $\widetilde{x}_{t}^\psi$ with $y_{t}$, $\widecheck{y}_{t}^\psi$ and $\widetilde{y}_{t}^\psi$ respectively. 
Therefore, it suffices to obtain the upper bound for the drift term. 
Indeed, we have the following proposition. 
%%%%%%%%%%%%%%%%%%%%%%%%%%%%%%%%%%%%%%

%%%%%%%%%%%%%%%%%%%%%%%%%%%%%%%%%%%%%%
\begin{proposition}
\label{auxiliary}
If an auxiliary strategy sequence $\left\{y_t\right\}_{t\geqslant 1}$ determined by $\left\{\widehat{y}_t^*\right\}_{t\geqslant 1}$ is introduced in addition to the primary strategy sequence $\left\{x_t\right\}_{t\geqslant 1}$, then the term $\phi^\star\left(\xi\left\lVert x_t^*-\widehat{x}_t^*\right\rVert\right)$ in regret upper bounds is replaced by $\mathit{\Phi}_{\xi}\left(x_t^*,\widehat{x}_t^*\right)$, and $x_{t}$, $\widecheck{x}_{t}^\psi$ and $\widetilde{x}_{t}^\psi$ in remaining terms are replaced by $y_{t}$, $\widecheck{y}_{t}^\psi$ and $\widetilde{y}_{t}^\psi$ respectively, where 
\begin{equation*}
\begin{aligned}
\mathit{\Phi}_{\xi}\left(x^*,\widehat{x}^*\right)=\phi^\star\left(\xi\left\lVert x^*-\widehat{y}^*\right\rVert\right)+\inf_{\gamma>0}\left(\frac{1}{\gamma}\phi^\star\left(\gamma\xi\left\lVert x^*\right\rVert\right)+\frac{1}{\gamma}\phi^\star\left(\xi\left\lVert \widehat{x}^*-\widehat{y}^*\right\rVert\right)\right).
\end{aligned}
\end{equation*}
If $\phi=Q_\rho$, and $\widehat{y}^*=\lambda\widehat{x}^*+\left(1-\lambda\right)x^*$, where $\lambda=\min\left\{\frac{\left\lVert x^*\right\rVert}{\left\lVert x^*-\widehat{x}^*\right\rVert},\,1\right\}$, then 
\begin{equation}
\label{Phi-term}
\begin{aligned}
\mathit{\Phi}_{\xi}\left(x^*,\widehat{x}^*\right)
&=Q_\rho^\star\big(\xi\min\left\{\left\lVert x^*-\widehat{x}^*\right\rVert, \left\lVert x^*\right\rVert\right\}\big)
+\xi\left\lVert x^*\right\rVert\min\left\{\xi\left(\left\lVert x^*-\widehat{x}^*\right\rVert-\left\lVert x^*\right\rVert\right)_+, \rho\right\} \\
&\leqslant\xi^2 Q_{\left\lVert x^*\right\rVert}^\star\left(\left\lVert x^*-\widehat{x}^*\right\rVert\right)-\frac{1}{2}\big(\xi\min\left\{\left\lVert x^*-\widehat{x}^*\right\rVert, \left\lVert x^*\right\rVert\right\}-\rho\big)_+^2 .
\end{aligned}
\end{equation}
\end{proposition}
%%%%%%%%%%%%%%%%%%%%%%%%%%%%%%%%%%%%%%

%%%%%%%%%%%%%%%%%%%%%%%%%%%%%%%%%%%%%%
\begin{remark}
$Q_\rho$-convex tightens the regret upper bound more finely. 
The Huber penalty in \cite{flaspohler2021online} corresponds to the $Q_\infty$-convex case, that is, the last subtraction term in the second line of \cref{Phi-term} vanishes. 
\end{remark}
%%%%%%%%%%%%%%%%%%%%%%%%%%%%%%%%%%%%%%

%%%%%%%%%%%%%%%%%%%%%%%%%%%%%%%%%%%%%%
\section{Instantiations}
\label{sec:instantiations}
%%%%%%%%%%%%%%%%%%%%%%%%%%%%%%%%%%%%%%

%%%%%%%%%%%%%%%%%%%%%%%%%%%%%%%%%%%%%%
All the above analyses are based on the abstract description of the $\phi$-convexity of $\psi$. 
In this section, we instantiate $\psi$ in two forms,  the negative entropy and the squared norm. 
%%%%%%%%%%%%%%%%%%%%%%%%%%%%%%%%%%%%%%

%%%%%%%%%%%%%%%%%%%%%%%%%%%%%%%%%%%%%%
\boldmath
\subsection{$\psi$ is Negative Entropy}\unboldmath
%%%%%%%%%%%%%%%%%%%%%%%%%%%%%%%%%%%%%%

%%%%%%%%%%%%%%%%%%%%%%%%%%%%%%%%%%%%%%
In this part, we set $E=\left(\mathbb{R}^{n+1}, \left\lVert\cdot\right\rVert_1\right)$, and then $E^*=\left(\mathbb{R}^{n+1}, \left\lVert\cdot\right\rVert_{\infty}\right)$. 
%%%%%%%%%%%%%%%%%%%%%%%%%%%%%%%%%%%%%%

%%%%%%%%%%%%%%%%%%%%%%%%%%%%%%%%%%%%%%
\begin{lemma}
\label{Negative-Entropy}
Set $\psi$ be the negative entropy, that is,
\begin{equation*}
\begin{aligned}
\psi\left(w\right)=\left\langle w, \ln w\right\rangle+\chi_{\bigtriangleup^{n}}\left(w\right),
\end{aligned}
\end{equation*}
where the probability simplex $\bigtriangleup^{n}\coloneqq\left\{w\in E\left\lvert\,w\geqslant 0,\,\left\lVert w\right\rVert_1=1\right.\right\}$ is a compact subset of $E$. 
Then $\psi$ is $Q_2$-convex, and the strategy $S$ can be instantiated as the following Optimistic Normalized Exponentiated Subgradient (\textup{ONES}),
\begin{equation*}
w_{t}=\mathscr{N}\left(a\circ\mathrm{e}^{-\eta_t\left(\sum_{i=1}^{t-1}\theta_i \ell_i+\theta_t\widehat{\ell}_t\right)}\right),
\end{equation*}
where $\mathscr{N}$ is the normalization operator, 
$\circ$ denotes the Hadamard product, 
$\ell_i\in E^*$ is the loss vector, 
$\widehat{\ell}_t$ is the estimated vector corresponding to $\ell_t$. 
\end{lemma}
%%%%%%%%%%%%%%%%%%%%%%%%%%%%%%%%%%%%%%

%%%%%%%%%%%%%%%%%%%%%%%%%%%%%%%%%%%%%%
$\textup{ONES}\big|_{\theta_t\equiv 1}$ can be rearranged as
\begin{equation*}
\begin{aligned}
\widetilde{w}_{t}&=\mathscr{N}\left(a\circ\mathrm{e}^{-\eta_t\sum_{i=1}^{t-1}\ell_i}\right), &&\quad a\in\bigtriangleup^{n},\\
w_{t}&=\mathscr{N}\left(\widetilde{w}_{t}\circ\mathrm{e}^{-\eta_t\widehat{\ell}_t}\right), &&
\end{aligned}
\end{equation*}
and its function form is $\left(w_{t}, \widetilde{w}_{t}\right)=\textup{ONES}\left(\widehat{\ell}_{t}, \ell_{t-1}; \eta_t, 1\right)$. 
$\textup{ONES}\big|_{\eta_t\equiv 1}$ can be rearranged as
\begin{equation*}
\begin{aligned}
\widetilde{w}_{t}&=\mathscr{N}\left(\widetilde{w}_{t-1}\circ\mathrm{e}^{-\theta_{t-1}\ell_{t-1}}\right), &&\quad a\in\bigtriangleup^{n},\\
w_{t}&=\mathscr{N}\left(\widetilde{w}_{t}\circ\mathrm{e}^{-\theta_t\widehat{\ell}_t}\right), &&
\end{aligned}
\end{equation*}
and its corresponding function form is $\left(w_{t}, \widetilde{w}_{t}\right)=\textup{ONES}\left(\widehat{\ell}_{t}, \ell_{t-1}; 1, \theta_t\right)$. 
%%%%%%%%%%%%%%%%%%%%%%%%%%%%%%%%%%%%%%

%%%%%%%%%%%%%%%%%%%%%%%%%%%%%%%%%%%%%%
\begin{remark}
It is worth mentioning that \textup{ONES} has two parameters, and $\textup{ONES}\big|_{\eta_t\equiv 1}$ has a plethora of different names, such as Optimistic Hedge, Optimistic Exponential Weights, etc. 
\end{remark}
%%%%%%%%%%%%%%%%%%%%%%%%%%%%%%%%%%%%%%

%%%%%%%%%%%%%%%%%%%%%%%%%%%%%%%%%%%%%%
The static regret for \textup{ONES} is formalized into the following corollary. 
%%%%%%%%%%%%%%%%%%%%%%%%%%%%%%%%%%%%%%

%%%%%%%%%%%%%%%%%%%%%%%%%%%%%%%%%%%%%%
\begin{corollary}[Static Regret for \textup{ONES}]
\label{ONES} 
Under the assumption of $\eta_t\geqslant\eta_{t+1}$, $\textup{ONES}\big|_{\theta_t\equiv 1}$ enjoys the following static regret upper bound,
\begin{equation}
\label{ONES-eta-regret}
\regret\left(u,u,\cdots,u\right)
\leqslant \frac{1}{\eta_{T+1}}\left\langle u, \ln \frac{u}{a}\right\rangle+\sum_{t=1}^{T}\frac{1}{\eta_t} Q_{2}^\star\left(\eta_t\left\lVert \ell_t-\widehat{\ell}_t\right\rVert_{\infty}\right)-\sum_{t=1}^T\frac{1}{\eta_t}\left\langle w_t, \ln \frac{w_t}{\widetilde{w}_t}\right\rangle,
\end{equation}
or the corresponding version with auxiliary strategies, 
\begin{equation}
\label{ONES-eta-regret-auxiliary}
\regret\left(u,u,\cdots,u\right)
\leqslant \frac{1}{\eta_{T+1}}\left\langle u, \ln \frac{u}{a}\right\rangle+\sum_{t=1}^{T}\frac{1}{\eta_t}\mathit{\Phi}_{\eta_t}\left(\ell_t,\widehat{\ell}_t\right)-\sum_{t=1}^T\frac{1}{\eta_t}\left\langle v_t, \ln \frac{v_t}{\widetilde{v}_t}\right\rangle,
\end{equation}
where $\forall u\in\bigtriangleup^{n}$, $\left(v_{t}, \widetilde{v}_{t}\right)=\textup{ONES}\left(\widehat{\ell}'_{t}, \ell_{t-1}; \eta_t, 1\right)$, $\widehat{\ell}'_{t}=\lambda\widehat{\ell}_t+\left(1-\lambda\right)\ell_t$, $\lambda=\min\left\{\frac{\left\lVert \ell_t\right\rVert_{\infty}}{\left\lVert \ell_t-\widehat{\ell}_t\right\rVert_{\infty}},\,1\right\}$, and $\mathit{\Phi}$ follows from \cref{Phi-term} with $\rho=2$. 

Under the assumption of $\theta_{t-1}\leqslant \theta_{t}$, $\textup{ONES}\big|_{\eta_t\equiv 1}$ enjoys the following static regret upper bound,  
\begin{equation}
\label{ONES-theta-regret}
\regret\left(u,u,\cdots,u\right)
\leqslant \frac{1}{\theta_{1}}\left\langle u, \ln \frac{u}{a}\right\rangle+\sum_{t=1}^{T}\frac{1}{\theta_t} Q_{2}^\star\left(\theta_t\left\lVert \ell_t-\widehat{\ell}_t\right\rVert_{\infty}\right)-\sum_{t=1}^T\frac{1}{\theta_t}\left\langle w_t, \ln \frac{w_t}{\widetilde{w}_t}\right\rangle,
\end{equation}
or the corresponding version with auxiliary strategies, 
\begin{equation}
\label{ONES-theta-regret-auxiliary}
\regret\left(u,u,\cdots,u\right)
\leqslant \frac{1}{\theta_{1}}\left\langle u, \ln \frac{u}{a}\right\rangle+\sum_{t=1}^{T}\frac{1}{\theta_t}\mathit{\Phi}_{\theta_t}\left(\ell_t,\widehat{\ell}_t\right)-\sum_{t=1}^T\frac{1}{\theta_t}\left\langle v_t, \ln \frac{v_t}{\widetilde{v}_t}\right\rangle,
\end{equation}
where $\forall u\in\bigtriangleup^{n}$, $\left(v_{t}, \widetilde{v}_{t}\right)=\textup{ONES}\left(\widehat{\ell}'_{t}, \ell_{t-1}; 1, \theta_t\right)$, $\widehat{\ell}'_{t}=\lambda\widehat{\ell}_t+\left(1-\lambda\right)\ell_t$, $\lambda=\min\left\{\frac{\left\lVert \ell_t\right\rVert_{\infty}}{\left\lVert \ell_t-\widehat{\ell}_t\right\rVert_{\infty}},\,1\right\}$, and $\mathit{\Phi}$ follows from \cref{Phi-term} with $\rho=2$. 
\end{corollary}
%%%%%%%%%%%%%%%%%%%%%%%%%%%%%%%%%%%%%%

%%%%%%%%%%%%%%%%%%%%%%%%%%%%%%%%%%%%%%
\begin{remark}
According to \cref{Entropy-strongly-convex}, 
\begin{equation*}
\begin{aligned}
\left\langle w_t, \ln \frac{w_t}{\widetilde{w}_t}\right\rangle\geqslant\frac{1}{2}\left\lVert w_t-\widetilde{w}_t\right\rVert_{1}^2, 
\end{aligned}
\end{equation*}
then \cref{ONES-eta-regret} (as well as Equation~\ref{ONES-theta-regret}) can be relaxed as 
\begin{equation*}
\begin{aligned}
\frac{1}{\eta_{T+1}}\left\langle u, \ln \frac{u}{a}\right\rangle+\sum_{t=1}^{T}\frac{\eta_t}{2}\left\lVert \ell_t-\widehat{\ell}_t\right\rVert_{\infty}^2-\sum_{t=1}^T\frac{1}{2\eta_t}\left(\eta_t\left\lVert \ell_t-\widehat{\ell}_t\right\rVert_{\infty}-2\right)_+^2-\sum_{t=1}^T\frac{1}{2\eta_t}\left\lVert w_t-\widetilde{w}_t\right\rVert_{1}^2. 
\end{aligned}
\end{equation*}
Without the first subtraction term, this bound is a refined version of Theorem~19 of \cite{syrgkanis2015fast},  
without all the subtraction terms, this bound is directly implied by Theorem~7.28 of \cite{orabona2020modern}. 

By dropping the last term $\left\langle v_t, \ln \frac{v_t}{\widetilde{v}_t}\right\rangle$, \cref{ONES-eta-regret-auxiliary} (as well as Equation~\ref{ONES-theta-regret-auxiliary}) can be relaxed as 
\begin{equation*}
\begin{aligned}
\frac{1}{\eta_{T+1}}\left\langle u, \ln \frac{u}{a}\right\rangle+\sum_{t=1}^{T}\eta_t Q_{\left\lVert \ell_t\right\rVert_{\infty}}^\star\left(\left\lVert \ell_t-\widehat{\ell}_t\right\rVert_{\infty}\right)
-\frac{1}{2}\sum_{t=1}^{T}\left(\eta_t\min\left\{\left\lVert \ell_t-\widehat{\ell}_t\right\rVert_{\infty}, \left\lVert  \ell_t\right\rVert_{\infty}\right\}-2\right)_+^2.
\end{aligned}
\end{equation*}
Without the last subtraction term, this bound is directly implied by Theorem~3 of \cite{flaspohler2021online}. 
\end{remark}
%%%%%%%%%%%%%%%%%%%%%%%%%%%%%%%%%%%%%%

%%%%%%%%%%%%%%%%%%%%%%%%%%%%%%%%%%%%%%
\boldmath
\subsection{$\psi$ is Squared Norm}\unboldmath
\label{sec:SN}
%%%%%%%%%%%%%%%%%%%%%%%%%%%%%%%%%%%%%%

%%%%%%%%%%%%%%%%%%%%%%%%%%%%%%%%%%%%%%
In this part, we set $E$ be a Hilbert space over $\mathbb{R}$. 
%%%%%%%%%%%%%%%%%%%%%%%%%%%%%%%%%%%%%%

%%%%%%%%%%%%%%%%%%%%%%%%%%%%%%%%%%%%%%
\begin{lemma}
\label{squared-norm}
Set $\psi$ be the squared norm on $C$, that is,
\begin{equation*}
\begin{aligned}
\psi\left(x\right)=\frac{1}{2}\left\lVert x\right\rVert^2+\chi_C\left(x\right),
\end{aligned}
\end{equation*}
and let $\rho=\sup_{x,y\in C}\left\lVert x-y\right\rVert$.  
Then $\psi$ is $Q_\rho$-convex, $\partial\psi\left(C\right)=E$, 
and $\si$ can be instantiated as the following Optimistic Lazy Projection (\textup{OLP}),
\begin{equation*}
\begin{aligned}
\widetilde{x}_t&=P_C\left(a-\eta_t\sum_{i=1}^{t-1}\theta_i x_i^*\right), \\
x_t&=P_C\left(\widetilde{x}_t -\eta_t\theta_t \widehat{x}_t^*\right),
\end{aligned}
\end{equation*}
$\sii$ can be instantiated as the following Optimistic Greedy Projection (\textup{OGP}),
\begin{equation*}
\begin{aligned}
\widetilde{x}_t&=P_C\left(\widetilde{x}_{t-1}-\theta_{t-1} x_{t-1}^*\right), \\
x_t&=P_C\left(\widetilde{x}_t -\theta_t \widehat{x}_t^*\right),
\end{aligned}
\end{equation*}
where $P_C$ represents the projection onto the closed convex subset $C$. 
\end{lemma}
%%%%%%%%%%%%%%%%%%%%%%%%%%%%%%%%%%%%%%

%%%%%%%%%%%%%%%%%%%%%%%%%%%%%%%%%%%%%%
The function form of \textup{OLP} can be written as $\left(x_{t}, \widetilde{x}_{t}\right)=\textup{OLP}\left(\widehat{x}_{t}^*, x_{t-1}^*; \eta_t, \theta_{t}\right)$, and the function form of \textup{OGP} can be written as $\left(x_{t}, \widetilde{x}_{t}\right)=\textup{OGP}\left(\widehat{x}_{t}^*, x_{t-1}^*; \theta_{t}\right)$. 
%%%%%%%%%%%%%%%%%%%%%%%%%%%%%%%%%%%%%%

%%%%%%%%%%%%%%%%%%%%%%%%%%%%%%%%%%%%%%
\begin{remark}
It is worth mentioning that the lazy projection here has two parameters, while the usual lazy projection has only one parameter. 
\cite{mcmahan2017survey} proved that the lazy projection algorithm is just DA and the greedy one is just MD. 
Here we proved that the two-parameter strategy \textup{OLP} is just $\si$. 
It should be emphasized that all projection algorithms must be in Hilbert space, and both $\si$ and $\sii$ run in normed vector space. The instantiations require the extra constraint that $\partial\varphi$ is the identity map. 
\end{remark}
%%%%%%%%%%%%%%%%%%%%%%%%%%%%%%%%%%%%%%

%%%%%%%%%%%%%%%%%%%%%%%%%%%%%%%%%%%%%%
\begin{corollary}[Dynamic Regret for \textup{OLP}]
\label{OLP}
Under the assumption of $\theta_{t-1}\leqslant \theta_{t}$, $\textup{OLP}\big|_{\eta_t\equiv 1}$ enjoys the following dynamic regret upper bound,
\begin{equation*}
\begin{aligned}
\regret\left(z_1,z_2,\cdots,z_T\right)
\leqslant\ &\frac{1}{2\theta_{1}}\left\lVert z_1-a\right\rVert^2+\sum_{t=2}^{T}\frac{1}{\theta_t}\left\lVert z_t-a+\sum_{i=1}^{t-1}\theta_i x_i^*\right\rVert \left\lVert z_t-z_{t-1}\right\rVert \\
&+\sum_{t=1}^{T}\frac{1}{\theta_t} Q_{\rho}^\star\left(\theta_t\left\lVert x_{t}^*-\widehat{x}_{t}^*\right\rVert\right)
-\sum_{t=1}^{T}\frac{1}{2\theta_t}\left\lVert x_{t}-\widetilde{x}_{t}\right\rVert^2,\quad\forall z_t\in C,
\end{aligned}
\end{equation*}
or the corresponding version with auxiliary strategies, 
\begin{equation*}
\begin{aligned}
\regret\left(z_1,z_2,\cdots,z_T\right)
\leqslant\ &\frac{1}{2\theta_{1}}\left\lVert z_1-a\right\rVert^2+\sum_{t=2}^{T}\frac{1}{\theta_t}\left\lVert z_t-a+\sum_{i=1}^{t-1}\theta_i x_i^*\right\rVert \left\lVert z_t-z_{t-1}\right\rVert \\
&+\sum_{t=1}^{T}\frac{1}{\theta_t}\mathit{\Phi}_{\theta_t}\left(x_t^*,\widehat{x}_t^*\right)
-\sum_{t=1}^{T}\frac{1}{2\theta_t}\left\lVert y_{t}-\widetilde{y}_{t}\right\rVert^2,\quad\forall z_t\in C,
\end{aligned}
\end{equation*}
where $\left(y_{t}, \widetilde{y}_{t}\right)=\textup{OLP}\left(\widehat{y}_{t}^*, x_{t-1}^*; 1, \theta_t\right)$, $\widehat{y}_t^*=\lambda\widehat{x}_t^*+\left(1-\lambda\right)x_t^*$, $\lambda=\min\left\{\frac{\left\lVert x_t^*\right\rVert}{\left\lVert x_t^*-\widehat{x}_t^*\right\rVert},\,1\right\}$, and $\mathit{\Phi}$ follows from \cref{Phi-term}. 
\end{corollary}
%%%%%%%%%%%%%%%%%%%%%%%%%%%%%%%%%%%%%%

%%%%%%%%%%%%%%%%%%%%%%%%%%%%%%%%%%%%%%
\begin{corollary}[Static Regret for \textup{OLP}]
\label{OLP-static}
Under the assumption of $\eta_t\geqslant\eta_{t+1}$, $\textup{OLP}\big|_{\theta_t\equiv 1}$ enjoys the following static regret upper bound,
\begin{equation*}
\begin{aligned}
\regret\left(z,z,\cdots,z\right)
\leqslant\frac{1}{2\eta_{T+1}}\left\lVert z-a\right\rVert^2+\sum_{t=1}^{T}\frac{1}{\eta_t} Q_{\rho}^\star\left(\eta_t\left\lVert x_{t}^*-\widehat{x}_{t}^*\right\rVert\right)
-\sum_{t=1}^{T}\frac{1}{2\eta_t}\left\lVert x_{t}-\widetilde{x}_{t}\right\rVert^2,
\end{aligned}
\end{equation*}
or the corresponding version with auxiliary strategies, 
\begin{equation*}
\begin{aligned}
\regret\left(z,z,\cdots,z\right)
\leqslant\frac{1}{2\eta_{T+1}}\left\lVert z-a\right\rVert^2+\sum_{t=1}^{T}\frac{1}{\eta_t}\mathit{\Phi}_{\eta_t}\left(x_t^*,\widehat{x}_t^*\right)
-\sum_{t=1}^{T}\frac{1}{2\eta_t}\left\lVert y_{t}-\widetilde{y}_{t}\right\rVert^2,
\end{aligned}
\end{equation*}
where $\forall z\in C$, $\left(y_{t}, \widetilde{y}_{t}\right)=\textup{OLP}\left(\widehat{y}_{t}^*, x_{t-1}^*; \eta_t, 1\right)$, $\widehat{y}_t^*=\lambda\widehat{x}_t^*+\left(1-\lambda\right)x_t^*$, $\lambda=\min\left\{\frac{\left\lVert x_t^*\right\rVert}{\left\lVert x_t^*-\widehat{x}_t^*\right\rVert},\,1\right\}$, and  $\mathit{\Phi}$ follows from \cref{Phi-term}. 

Under the assumption of $\theta_{t-1}\leqslant \theta_{t}$, $\textup{OLP}\big|_{\eta_t\equiv 1}$ enjoys the following static regret upper bound,
\begin{equation*}
\begin{aligned}
\regret\left(z,z,\cdots,z\right)
\leqslant\frac{1}{2\theta_{1}}\left\lVert z-a\right\rVert^2+\sum_{t=1}^{T}\frac{1}{\theta_t} Q_{\rho}^\star\left(\theta_t\left\lVert x_{t}^*-\widehat{x}_{t}^*\right\rVert\right)
-\sum_{t=1}^{T}\frac{1}{2\theta_t}\left\lVert x_{t}-\widetilde{x}_{t}\right\rVert^2,
\end{aligned}
\end{equation*}
or the corresponding version with auxiliary strategies, 
\begin{equation*}
\begin{aligned}
\regret\left(z,z,\cdots,z\right)
\leqslant\frac{1}{2\theta_{1}}\left\lVert z-a\right\rVert^2+\sum_{t=1}^{T}\frac{1}{\theta_t}\mathit{\Phi}_{\theta_t}\left(x_t^*,\widehat{x}_t^*\right)
-\sum_{t=1}^{T}\frac{1}{2\theta_t}\left\lVert y_{t}-\widetilde{y}_{t}\right\rVert^2,
\end{aligned}
\end{equation*}
where $\forall z\in C$, $\left(y_{t}, \widetilde{y}_{t}\right)=\textup{OLP}\left(\widehat{y}_{t}^*, x_{t-1}^*; 1, \theta_t\right)$, $\widehat{y}_t^*=\lambda\widehat{x}_t^*+\left(1-\lambda\right)x_t^*$, $\lambda=\min\left\{\frac{\left\lVert x_t^*\right\rVert}{\left\lVert x_t^*-\widehat{x}_t^*\right\rVert},\,1\right\}$, and  $\mathit{\Phi}$ follows from \cref{Phi-term}. 
\end{corollary}
%%%%%%%%%%%%%%%%%%%%%%%%%%%%%%%%%%%%%%

%%%%%%%%%%%%%%%%%%%%%%%%%%%%%%%%%%%%%%
\begin{corollary}[Dynamic Regret for \textup{OGP}]
\label{OGP}
Under the assumption of $\theta_{t-1}\leqslant \theta_{t}$, \textup{OGP} enjoys the following dynamic regret upper bound,  
\begin{equation*}
\begin{aligned}
\regret\left(z_1,z_2,\cdots,z_T\right)
\leqslant\ &\frac{1}{2\theta_{1}}\left\lVert z_1-a\right\rVert^2+\sum_{t=2}^{T}\frac{\rho}{\theta_t}\left\lVert z_t-z_{t-1}\right\rVert^2 \\
&+\sum_{t=1}^{T}\frac{1}{\theta_t} Q_{\rho}^\star\left(\theta_t\left\lVert x_{t}^*-\widehat{x}_{t}^*\right\rVert\right)
-\sum_{t=1}^{T}\frac{1}{2\theta_t}\left\lVert x_{t}-\widetilde{x}_{t}\right\rVert^2,\quad\forall z_t\in C,
\end{aligned}
\end{equation*}
or the corresponding version with auxiliary strategies, 
\begin{equation*}
\begin{aligned}
\regret\left(z_1,z_2,\cdots,z_T\right)
\leqslant\ &\frac{1}{2\theta_{1}}\left\lVert z_1-a\right\rVert^2+\sum_{t=2}^{T}\frac{\rho}{\theta_t}\left\lVert z_t-z_{t-1}\right\rVert^2 \\
&+\sum_{t=1}^{T}\frac{1}{\theta_t}\mathit{\Phi}_{\theta_t}\left(x_t^*,\widehat{x}_t^*\right)
-\sum_{t=1}^{T}\frac{1}{2\theta_t}\left\lVert y_{t}-\widetilde{y}_{t}\right\rVert^2,\quad\forall z_t\in C,
\end{aligned}
\end{equation*}
where $\big(y_{t}, \widetilde{y}_{t}\big)=\textup{OGP}\left(\widehat{y}_{t}^*, x_{t-1}^*; \theta_t\right)$, $\widehat{y}_t^*=\lambda\widehat{x}_t^*+\left(1-\lambda\right)x_t^*$, $\lambda=\min\left\{\frac{\left\lVert x_t^*\right\rVert}{\left\lVert x_t^*-\widehat{x}_t^*\right\rVert},\,1\right\}$, and $\mathit{\Phi}$ follows from \cref{Phi-term}. 
\end{corollary}
%%%%%%%%%%%%%%%%%%%%%%%%%%%%%%%%%%%%%%

%%%%%%%%%%%%%%%%%%%%%%%%%%%%%%%%%%%%%%
\begin{corollary}[Static Regret for \textup{OGP}]
\label{OGP-static}
Under the assumption of $\theta_{t-1}\leqslant \theta_{t}$, \textup{OGP} enjoys the following static regret upper bound,
\begin{equation*}
\begin{aligned}
\regret\left(z,z,\cdots,z\right)
\leqslant\frac{1}{2\theta_{1}}\left\lVert z-a\right\rVert^2+\sum_{t=1}^{T}\frac{1}{\theta_t} Q_{\rho}^\star\left(\theta_t\left\lVert x_{t}^*-\widehat{x}_{t}^*\right\rVert\right)
-\sum_{t=1}^{T}\frac{1}{2\theta_t}\left\lVert x_{t}-\widetilde{x}_{t}\right\rVert^2,
\end{aligned}
\end{equation*}
or the corresponding version with auxiliary strategies, 
\begin{equation*}
\begin{aligned}
\regret\left(z,z,\cdots,z\right)
\leqslant\frac{1}{2\theta_{1}}\left\lVert z-a\right\rVert^2+\sum_{t=1}^{T}\frac{1}{\theta_t}\mathit{\Phi}_{\theta_t}\left(x_t^*,\widehat{x}_t^*\right)
-\sum_{t=1}^{T}\frac{1}{2\theta_t}\left\lVert y_{t}-\widetilde{y}_{t}\right\rVert^2,
\end{aligned}
\end{equation*}
where $\forall z\in C$, $\left(y_{t}, \widetilde{y}_{t}\right)=\textup{OGP}\left(\widehat{y}_{t}^*, x_{t-1}^*; 1, \theta_t\right)$, $\widehat{y}_t^*=\lambda\widehat{x}_t^*+\left(1-\lambda\right)x_t^*$, $\lambda=\min\left\{\frac{\left\lVert x_t^*\right\rVert}{\left\lVert x_t^*-\widehat{x}_t^*\right\rVert},\,1\right\}$, and  $\mathit{\Phi}$ follows from \cref{Phi-term}. 
\end{corollary}
%%%%%%%%%%%%%%%%%%%%%%%%%%%%%%%%%%%%%%

%%%%%%%%%%%%%%%%%%%%%%%%%%%%%%%%%%%%%%
\cref{pf:OLP} provides the proof of \cref{OLP}. 
\cref{OGP} can be proved in the same way. 
\cref{OLP-static,OGP-static} can be obtained by setting $z_t\equiv z$ directly. 
%%%%%%%%%%%%%%%%%%%%%%%%%%%%%%%%%%%%%%

%%%%%%%%%%%%%%%%%%%%%%%%%%%%%%%%%%%%%%
\begin{remark}
The regret bounds of the above four corollaries are better than the currently known optimal results. This is due to the fact that the last term of all upper bounds is the extra subtraction term and the more refined choice of $Q_\rho$-convexity. 
\end{remark}
%%%%%%%%%%%%%%%%%%%%%%%%%%%%%%%%%%%%%%

%%%%%%%%%%%%%%%%%%%%%%%%%%%%%%%%%%%%%%

%%%%%%%%%%%%%%%%%%%%%%%%%%%%%%%%%%%%%%

%%%%%%%%%%%%%%%%%%%%%%%%%%%%%%%%%%%%%%
\section{Online Monotone Optimization}
%%%%%%%%%%%%%%%%%%%%%%%%%%%%%%%%%%%%%%

%%%%%%%%%%%%%%%%%%%%%%%%%%%%%%%%%%%%%%
In this section, we argue monotonicity rather than convexity
is the natural boundary for $\si$ and $\sii$. 
By allowing the absence of loss functions, we extend the scope of application of all update rules to online monotone optimization. 
Although \cite{gemp2016online} proposed the idea of online monotone optimization, we emphasize that the concept of online monotone optimization needs to be rigorously reformulated. % under the unified framework and based on recent research. 
See \cite{romano1993potential} for the potential theory of monotone multivalued operators. 
Part of the relevant basis is compiled in \cref{BoMO}. 
%%%%%%%%%%%%%%%%%%%%%%%%%%%%%%%%%%%%%%

%%%%%%%%%%%%%%%%%%%%%%%%%%%%%%%%%%%%%%
Let $E$ be a normed vector space over $\mathbb{R}$ and let $C\neq\varnothing$ be a closed convex subset of $E$. 
The online monotone optimization problem can be formalized as follows. At round $t$, 
\begin{equation*}
\begin{aligned}
&\text{the player chooses }x_t\in C\text{ according to some algorithm, }\\ 
&\text{the adversary (environment) feeds back a monotone operator }M_t\colon\! C\rightrightarrows E^*, 
\end{aligned}
\end{equation*}
where ``$\rightrightarrows$'' emphasizes that the map is multivalued. 
We choose the following generalized dynamic regret as the performance metric. 
%%%%%%%%%%%%%%%%%%%%%%%%%%%%%%%%%%%%%%
%
%%%%%%%%%%%%%%%%%%%%%%%%%%%%%%%%%%%%%%
\begin{definition}[$\regretn{n}$]
\label{generalized-regret}
The instantaneous generalized dynamic regret of strategy $x_t$ relative to the reference strategy $z_t$  in round $t$ is defined as
\begin{equation*}
\regretn{n} z_t\coloneqq\inf\int_{\gamma_n\left(z_t, x_t\right)}\left\langle M_t \alpha,\,d\alpha\right\rangle,\quad\forall n\in\mathbb{N},
\end{equation*}
where the infimum operator traverses all possible decompositions of 
$\gamma_n\left(z_t, x_t\right)=\overrightarrow{z_t \alpha_1}+ \overrightarrow{\alpha_1 \alpha_2}+ \cdots + \overrightarrow{\alpha_{n} x_t}$, $\forall \alpha_i\in C$, $\overrightarrow{\alpha x}$ represents the map $\lambda\mapsto\alpha+\lambda\left(x-\alpha\right)$, $ \lambda\in \left[0, 1\right]$, and if $n\in\mathbb{N}_+$, $\exists \mu\in\mathbb{R}$ such that $\forall x, y \in C$, $\left\langle M_t x, y-x\right\rangle\geqslant\mu$. 
\end{definition}
%%%%%%%%%%%%%%%%%%%%%%%%%%%%%%%%%%%%%%

%%%%%%%%%%%%%%%%%%%%%%%%%%%%%%%%%%%%%%
\begin{remark}
The online convex optimization problem is manifested as alternate actions of multivalued maps $\partial\varphi_t\colon\! C\rightrightarrows E^*$ and $\partial\psi^\star\colon\! E^*\rightrightarrows C$. 
Correspondingly, the multivalued maps for online monotone optimization are $M_t\colon\! C\rightrightarrows E^*$ and $\partial\psi^\star\colon\! E^*\rightrightarrows C$. 
Even if $\partial\psi^\star$ is replaced by a monotone operator, we can always choose a conservative monotone operator such that its potential is $\psi^\star$. 
In brief, online monotone optimization is to replace $\partial\varphi_t$ in online convex optimization with a monotone operator $M_t$. 
The generalized dynamic regret is a series of performance metrics determined by $n\in\mathbb{N}$, where $n$ represents the maximum meta-algorithm nesting level that the online algorithm can accommodate, and obviously, $\regretn{n}z_t\leqslant\regretn{n-1}z_t$, where $z_t$ represents the reference strategy in round $t$. 
\end{remark}
%%%%%%%%%%%%%%%%%%%%%%%%%%%%%%%%%%%%%%

%%%%%%%%%%%%%%%%%%%%%%%%%%%%%%%%%%%%%%
Next, we analyze the rationality of online monotone optimization. 
%%%%%%%%%%%%%%%%%%%%%%%%%%%%%%%%%%%%%%

%%%%%%%%%%%%%%%%%%%%%%%%%%%%%%%%%%%%%%
Firstly, \cref{generalized-regret} is well-defined since for any fixed $n\in\mathbb{N}_+$, 
\begin{equation*}
\begin{aligned}
\int_{\gamma_n\left(z_t, x_t\right)}\left\langle M_t \alpha,\,d\alpha\right\rangle
=\left(\int_{\overrightarrow{z_t \alpha_1}}+\int_{\overrightarrow{\alpha_1 \alpha_2}} + \cdots + \int_{\overrightarrow{\alpha_{n} x_t}}\right)\left\langle M_t \alpha,\,d\alpha\right\rangle
\geqslant\left(n+1\right)\mu.
\end{aligned}
\end{equation*}
%%%%%%%%%%%%%%%%%%%%%%%%%%%%%%%%%%%%%%

%%%%%%%%%%%%%%%%%%%%%%%%%%%%%%%%%%%%%%
Secondly, the definition of generalized dynamic regret extends the original definition of dynamic regret. 
%%%%%%%%%%%%%%%%%%%%%%%%%%%%%%%%%%%%%%

%%%%%%%%%%%%%%%%%%%%%%%%%%%%%%%%%%%%%%
On one hand, if the feedback is a loss function (we treat the restricted loss function $\varphi_t+\chi_C$ as the loss function, and for simplicity, we still denote it as $\varphi_t$), 
then according to \cref{Potential}, the potential of $M_t=\partial\varphi_t$ can be chosen as $\varphi_t$, and 
\begin{equation*}
\begin{aligned}
\varphi_t\left(x_t\right)-\varphi_t\left(z_t\right)=\int_{\gamma\left(z_t,x_t\right)}\left\langle\partial\varphi_t\left(\alpha\right),\,d\alpha\right\rangle
=\regretn{n} z_t,
\end{aligned}
\end{equation*}
where $\gamma\left(z_t,x_t\right)$ denotes an arbitrary finite-length oriented polyline in $C$ from $z_t$ to $x_t$, and the integral is path-independent since $\partial\varphi_t$ is conservative. 
%%%%%%%%%%%%%%%%%%%%%%%%%%%%%%%%%%%%%%

%%%%%%%%%%%%%%%%%%%%%%%%%%%%%%%%%%%%%%
On the other hand, if $M_t$ is a conservative monotone operator, and its potential is $\varphi_t$, then by \cref{def:Potential}, we have 
\begin{equation*}
\regretn{n}z_t=\int_{\gamma_n\left(z_t,x_t\right)}\left\langle M_t \alpha,\,d\alpha\right\rangle=\varphi_t\left(x_t\right)-\varphi_t\left(z_t\right),
\end{equation*}
and $\varphi_t$ is convex and lower semicontinuous according to \cref{Conservative-convex}. 
Typically in this case the superscript $n$ can be omitted. 
%%%%%%%%%%%%%%%%%%%%%%%%%%%%%%%%%%%%%%

%%%%%%%%%%%%%%%%%%%%%%%%%%%%%%%%%%%%%%
\begin{remark}
If $M_t$ is a non-conservative monotone operator, the integral is path-dependent, and its potential does not exist, that is, the loss function cannot be defined. 
Online monotone optimization allows situations where the loss function cannot be defined. 
This point of view is different from \cite{gemp2016online, gemp2017online, gemp2019optimization}, who proposed the idea of online monotone optimization, but did not to abandon the loss function. 
%Essentially, Online gaming should emphasize the interaction between the two parties and weaken the restrictions on the feedback of the adversary. 
In essence, the online game emphasizes the interaction between the two parties and weakens the restrictions on the feedback of the adversary. 
\end{remark}
%%%%%%%%%%%%%%%%%%%%%%%%%%%%%%%%%%%%%%

%%%%%%%%%%%%%%%%%%%%%%%%%%%%%%%%%%%%%%
\begin{remark}
For online convex optimization problem, we can regard the loss function as the potential and the regret as the potential difference. 
This perspective is guaranteed by strict mathematics. 
\end{remark}
%%%%%%%%%%%%%%%%%%%%%%%%%%%%%%%%%%%%%%

%%%%%%%%%%%%%%%%%%%%%%%%%%%%%%%%%%%%%%
Thirdly, surrogate linear losses are compatible with the settings of online monotone optimization. 
%%%%%%%%%%%%%%%%%%%%%%%%%%%%%%%%%%%%%%

%%%%%%%%%%%%%%%%%%%%%%%%%%%%%%%%%%%%%%
For online convex optimization problem, it suffices to study the corresponding online surrogate linear optimization problem. 
Note that  
\begin{equation}
\varphi_t\left(x_t\right)-\varphi_t\left(z_t\right)
\leqslant\left\langle x_t^*,x_t-z_t\right\rangle, \quad x_t^*\in\partial\varphi_t\left(x_t\right), \label{instantaneous}
\end{equation}
where $\partial\varphi_t$ plays the bridging role in transforming the online convex optimization problem into the corresponding online surrogate linear optimization problem. 
Since $\partial\varphi_t$ is monotone, we replace $\partial\varphi_t$ with a monotone operator $M_t\colon\! C\rightrightarrows E^*$. 
According to the definition of integral (\cref{def:Integrals}), $\forall n\in\mathbb{N}$, we have 
\begin{equation*}
\regretn{n}z_t\leqslant\int_{\overrightarrow{z_t x_t}}\left\langle M_t \alpha,\,d\alpha\right\rangle\leqslant\left\langle x_t^*, x_t-z_t\right\rangle,\quad x_t^*\in M_t x_t,
\end{equation*}
which is an extension of \cref{instantaneous}. 
%%%%%%%%%%%%%%%%%%%%%%%%%%%%%%%%%%%%%%

%%%%%%%%%%%%%%%%%%%%%%%%%%%%%%%%%%%%%%
Finally, the superscript $n$ represents the maximum meta-algorithm nesting level that the online algorithm can accommodate. 
%%%%%%%%%%%%%%%%%%%%%%%%%%%%%%%%%%%%%%

%%%%%%%%%%%%%%%%%%%%%%%%%%%%%%%%%%%%%%
In online meta-learning, the outermost meta-algorithm maintains a group of experts $\left\{e_i\right\}_{i\in I}$, and tracks the best one through the combination of expert advice using weight $w_t$. 
Note that $\forall n\in\mathbb{N}_+$, the instantaneous generalized dynamic regret can be decomposed as the following recursive inequality, 
\begin{equation*}
\begin{aligned}
\regretn{n}z_t
\leqslant \int_{\overrightarrow{x_t\left(j\right)\,\overline{x}_t}}\left\langle M_t \alpha,\,d\alpha\right\rangle + \inf\int_{\gamma_{n-1}\left(z_t, x_t\left(j\right)\right)}\left\langle M_t \alpha,\,d\alpha\right\rangle
\leqslant\underbrace{\left\langle \ell_t, w_t-1_j\right\rangle}_{\text{meta-regret}}+\underbrace{\regretn{n-1}z_t}_{e_j\text{'s regret}}, 
\end{aligned}
\end{equation*}
where $1_j$ is the one-hot vector corresponding to the expert $e_j$, $\ell_t\in\left\langle M_t\overline{x}_t, \boldsymbol{x}_t\right\rangle$, $\overline{x}_t=\left\langle w_t, \boldsymbol{x}_t\right\rangle$, $\boldsymbol{x}_t=\left\{x_t\left(i\right)\right\}_{i\in I}$, $x_t\left(i\right)$ represents the suggestion of the expert $e_i$. 
%%%%%%%%%%%%%%%%%%%%%%%%%%%%%%%%%%%%%%

%%%%%%%%%%%%%%%%%%%%%%%%%%%%%%%%%%%%%%
If $n-1>0$, then expert $e_j$ can still be a meta-algorithm, which causes the inequality to continue downward recursion. 
%%%%%%%%%%%%%%%%%%%%%%%%%%%%%%%%%%%%%%

%%%%%%%%%%%%%%%%%%%%%%%%%%%%%%%%%%%%%%
For example, $\forall n\in\mathbb{N}_+$, $\regretn{n}$ can accommodate one meta-algorithm nesting, that is, 
\begin{equation}
\label{regret1}
\begin{aligned}
\regretn{n}z_t
\leqslant\left(\int_{\overrightarrow{x_t\left(j\right)\,\overline{x}_t}} + \int_{\overrightarrow{z_t\, x_t\left(j\right)}}\right)\left\langle M_t \alpha,\,d\alpha\right\rangle 
\leqslant\left\langle \ell_t, w_t-1_j\right\rangle + \left\langle x_t\left(j\right)^*,x_t\left(j\right)-z_t\right\rangle, 
\end{aligned}
\end{equation}
where $x_t\left(j\right)^*\in M_t x_t\left(j\right)$. However, $\regretn{0} z_t$ has no corresponding decomposition, which implies that $\regretn{0}$ cannot accommodate any meta-algorithms. 
%%%%%%%%%%%%%%%%%%%%%%%%%%%%%%%%%%%%%%

%%%%%%%%%%%%%%%%%%%%%%%%%%%%%%%%%%%%%%
Corresponding to \cref{regret1}, in online convex optimization with meta-algorithm, the instantaneous dynamic regret is usually decomposed as the following form \citep{zhang2018adaptive}, 
\begin{equation*}
\begin{aligned}
\varphi_t\left(\overline{x}_t\right)-\varphi_t\left(z_t\right)
&=\varphi_t\left(\left\langle w_t,\boldsymbol{x}_t\right\rangle\right)-\varphi_t\left(\left\langle 1_j,\boldsymbol{x}_t\right\rangle\right)+\varphi_t\left(x_t\left(j\right)\right)-\varphi_t\left(z_t\right) \\
&\leqslant\left\langle \ell_t, w_t-1_j\right\rangle+\left\langle x_t\left(j\right)^*,x_t\left(j\right)-z_t\right\rangle, 
\end{aligned}
\end{equation*}
where $\ell_t\in\left\langle\partial\varphi_t\left(\overline{x}_t\right), \boldsymbol{x}_t\right\rangle$, and $x_t\left(j\right)^*\in\partial\varphi_t\left(x_t\left(j\right)\right)$. 
%%%%%%%%%%%%%%%%%%%%%%%%%%%%%%%%%%%%%%

%%%%%%%%%%%%%%%%%%%%%%%%%%%%%%%%%%%%%%
\section{Conclusions and Future Work}
%%%%%%%%%%%%%%%%%%%%%%%%%%%%%%%%%%%%%%

%%%%%%%%%%%%%%%%%%%%%%%%%%%%%%%%%%%%%%
In this paper, we present a unified analysis method for tighter dynamic regret upper bounds of $\si$ and $\sii$ in normed vector space, and extend online convex optimization to online monotone optimization, which expand the application scope of $\si$ and $\sii$. 
Our analysis is systematic and mathematically rigorous.
%%%%%%%%%%%%%%%%%%%%%%%%%%%%%%%%%%%%%%

%%%%%%%%%%%%%%%%%%%%%%%%%%%%%%%%%%%%%%
This paper only focuses on the online optimization problem of the learner suffering from the hitting cost. 
We leave the smoothing online optimization problem (in which the learner suffers both a hitting cost and a switching cost) to future research. 
Monotonic optimization is a natural extension from convex optimization to non-convex optimization. 
Exploring other non-convex optimization frameworks that are different from online monotone optimization is left to future work. 
We also hope that the conclusions of this article encourage the research of adaptive algorithms with tighter dynamic regret upper bounds. 
%%%%%%%%%%%%%%%%%%%%%%%%%%%%%%%%%%%%%%

%%%%%%%%%%%%%%%%%%%%%%%%%%%%%%%%%%%%%%
\clearpage
\bibliography{reference}
%%%%%%%%%%%%%%%%%%%%%%%%%%%%%%%%%%%%%%

%%%%%%%%%%%%%%%%%%%%%%%%%%%%%%%%%%%%%%
\clearpage
\appendix
%%%%%%%%%%%%%%%%%%%%%%%%%%%%%%%%%%%%%%

%%%%%%%%%%%%%%%%%%%%%%%%%%%%%%%%%%%%%%
\section{Basis of Convex Analysis}
\label{BoCA}
%%%%%%%%%%%%%%%%%%%%%%%%%%%%%%%%%%%%%%

%%%%%%%%%%%%%%%%%%%%%%%%%%%%%%%%%%%%%%
We use ``$\gra$'' to represent the graph of a map (function), and ``$\epi$'' to represent the epigraph of a map (function).
%%%%%%%%%%%%%%%%%%%%%%%%%%%%%%%%%%%%%%

%%%%%%%%%%%%%%%%%%%%%%%%%%%%%%%%%%%%%%
\cref{def:Convex-Function,def:Lower-Semicontinuous,def:Fenchel-Conjugate} and \cref{Fenchel-Moreau} are compiled from Section~1.4 of {\cite{brezis2010functional}}. 
%%%%%%%%%%%%%%%%%%%%%%%%%%%%%%%%%%%%%%

%%%%%%%%%%%%%%%%%%%%%%%%%%%%%%%%%%%%%%
\begin{definition}[Convex Function]
\label{def:Convex-Function}
A proper function $\varphi$ is convex if $\epi\varphi$ is a convex subset of $E\times\mathbb{R}$, or equivalently, $\dom\varphi\coloneqq\varphi^{-1}\mathbb{R}$ is a convex subset of $E$, and 
\begin{equation*}
\varphi\left(\lambda x+\left(1-\lambda\right)y\right)\leqslant\lambda \varphi\left(x\right)+\left(1-\lambda\right)\varphi\left(y\right),\quad\forall x,y\in\dom\varphi,\quad\forall \lambda\in\left[0,1\right].
\end{equation*} 
\end{definition}
%%%%%%%%%%%%%%%%%%%%%%%%%%%%%%%%%%%%%%

%%%%%%%%%%%%%%%%%%%%%%%%%%%%%%%%%%%%%%
\begin{definition}[Lower Semicontinuous]
\label{def:Lower-Semicontinuous}
A proper function $\varphi$ is lower semicontinuous if any one of the following three equivalent conditions holds: 

\textup{(a)} $\forall\lambda\in\mathbb{R}$, $\varphi^{-1}\left(\lambda, +\infty\right]$ is open. 

\textup{(b)} $\liminf_{x\rightarrow x_0}\varphi\left(x\right)\geqslant\varphi\left(x_0\right)$, $\forall x_0\in\dom\varphi$. 

\textup{(c)} $\epi\varphi$ is closed.
\end{definition}
%%%%%%%%%%%%%%%%%%%%%%%%%%%%%%%%%%%%%%

%%%%%%%%%%%%%%%%%%%%%%%%%%%%%%%%%%%%%%
\begin{definition}[Fenchel Conjugate]
\label{def:Fenchel-Conjugate}
Let $\varphi$ be a proper function. 
$\varphi^\star\left(x^*\right)\coloneqq\sup_E x^*-\varphi$ is the Fenchel conjugate function of $\varphi$, and $\varphi^\star$ is convex and lower semicontinuous. 
\end{definition}
%%%%%%%%%%%%%%%%%%%%%%%%%%%%%%%%%%%%%%

%%%%%%%%%%%%%%%%%%%%%%%%%%%%%%%%%%%%%%
\begin{lemma}
\label{subdifferential-convexity}
Let $C\neq\varnothing$ be a closed convex subset of $E$. 
If $\dom \varphi=\dom\partial \varphi=C$, then $\varphi$ is convex and lower semicontinuous. 
\end{lemma} 
%%%%%%%%%%%%%%%%%%%%%%%%%%%%%%%%%%%%%%

%%%%%%%%%%%%%%%%%%%%%%%%%%%%%%%%%%%%%%
\begin{proof}
Note that 
\begin{equation*}
\begin{aligned}
\varphi\left(y\right)\geqslant\left\langle x^*,y-x\right\rangle+\varphi\left(x\right),\quad\forall x\in C,\quad\forall x^*\in\partial \varphi\left(x\right),\quad\forall y\in E,
\end{aligned}
\end{equation*}
which implies that 
\begin{equation*}
\varphi\left(y\right)\geqslant\sup_{x\in C,\,x^*\in\partial \varphi\left(x\right)}\left\langle x^*,y-x\right\rangle+\varphi\left(x\right)\geqslant \varphi\left(y\right),\quad \forall y\in C, 
\end{equation*}
therefore we obtain $\varphi=f+\chi_{C}$, where 
\begin{equation*}
\begin{aligned}
f\left(y\right)=\sup_{x\in C,\,x^*\in\partial \varphi\left(x\right)}\ell_{x,x^*}\left(y\right),\quad\text{and}\quad
\ell_{x,x^*}\left(y\right)=\left\langle x^*,y-x\right\rangle+\varphi\left(x\right).
\end{aligned}
\end{equation*}
Note that 
\begin{equation*}
\epi f=\bigcap_{x\in C,\,x^*\in\partial\varphi\left(x\right)}\epi\ell_{x,x^*},
\end{equation*}
$\epi \varphi=\left(C\times E^*\right)\cap\epi f$ is closed and convex. 
According to \cref{def:Convex-Function,def:Lower-Semicontinuous}, we conclude that $\varphi$ is convex and lower semicontinuous. 
\end{proof}
%%%%%%%%%%%%%%%%%%%%%%%%%%%%%%%%%%%%%%

%%%%%%%%%%%%%%%%%%%%%%%%%%%%%%%%%%%%%%
\begin{lemma}[Fenchel-Moreau]
\label{Fenchel-Moreau}
If $\varphi$ is proper, convex and lower semicontinuous, then $\varphi^{\star\star}=\varphi$, where $\varphi^{\star\star}\left(x\right)\coloneqq\sup_{E^*}x-\varphi^{\star}$. 
\end{lemma}
%%%%%%%%%%%%%%%%%%%%%%%%%%%%%%%%%%%%%%

%%%%%%%%%%%%%%%%%%%%%%%%%%%%%%%%%%%%%%
\section{Proof of \cref{Generalized-cosine-law}}
%%%%%%%%%%%%%%%%%%%%%%%%%%%%%%%%%%%%%%

%%%%%%%%%%%%%%%%%%%%%%%%%%%%%%%%%%%%%%
\begin{proof}
\begin{equation*}
\begin{aligned}
 &B_{\varphi}\left(x, y^*\right)+B_{\varphi}\left(y, z^*\right)-B_{\varphi}\left(x, z^*\right) \\
=\ &\left(\varphi\left(x\right)+\varphi^\star\left(y^*\right)-\left\langle y^*, x\right\rangle\right)+\left(\varphi\left(y\right)+\varphi^\star\left(z^*\right)-\left\langle z^*, y\right\rangle\right)
-\left(\varphi\left(x\right)+\varphi^\star\left(z^*\right)-\left\langle z^*, x\right\rangle\right) \\
=\ &\varphi^\star\left(y^*\right)+\varphi\left(y\right)-\left\langle z^*, y\right\rangle+\left\langle z^*-y^*, x\right\rangle \\
=\ &\left\langle z^*-y^*, x-y\right\rangle,
\end{aligned}
\end{equation*}
where the last ``$=$'' follows from $\varphi^\star\left(y^*\right)+\varphi\left(y\right)=\left\langle y^*, y\right\rangle$ since $\left(y, y^*\right)\in\gra\partial\varphi$. 
\end{proof}
%%%%%%%%%%%%%%%%%%%%%%%%%%%%%%%%%%%%%%

%%%%%%%%%%%%%%%%%%%%%%%%%%%%%%%%%%%%%%
\section{Proof of \cref{Symmetry-of-subdifferential}}
%%%%%%%%%%%%%%%%%%%%%%%%%%%%%%%%%%%%%%

%%%%%%%%%%%%%%%%%%%%%%%%%%%%%%%%%%%%%%
\begin{proof}
Since $\varphi^{\star\star}=\varphi$ by \cref{Fenchel-Moreau}, we have
\begin{equation*}
\begin{aligned}
&x^*\in\partial\varphi\left(x\right)
\Longleftrightarrow x,x^* \text{ satisfies the equation }\varphi^\star\left(x^*\right)+\varphi\left(x\right)=\left\langle x^*,x\right\rangle
\Longleftrightarrow x\in\partial\varphi^\star\left(x^*\right).
\end{aligned}
\end{equation*}
%%%%%%%%%%%%%%%%%%%%%%%%%%%%%%%%%%%%%%

%%%%%%%%%%%%%%%%%%%%%%%%%%%%%%%%%%%%%%
Given $x^*$. On one hand, if $\partial\varphi^\star\left(x^*\right)\neq\varnothing$, then $\forall x\in\partial\varphi^\star\left(x^*\right)$, $\varphi^\star\left(x^*\right)=\left\langle x^*,x\right\rangle-\varphi\left(x\right)$ holds, that is, the equation $\left\langle x^*,x\right\rangle-\varphi\left(x\right)=\sup_{x\in E}\left\langle x^*,x\right\rangle-\varphi\left(x\right)$ holds, thus, $x\in\arg\max_{x\in E}\left\langle x^*,x\right\rangle-\varphi\left(x\right)$; 
If $\partial\varphi^\star\left(x^*\right)=\varnothing$, then $\forall x\in E$, $\varphi^\star\left(x^*\right)=\left\langle x^*,x\right\rangle-\varphi\left(x\right)$ does not hold, that is, the equation $\left\langle x^*,x\right\rangle-\varphi\left(x\right)=\sup_{x\in E}\left\langle x^*,x\right\rangle-\varphi\left(x\right)$ does not hold, thus $\arg\max_{x\in E}\left\langle x^*,x\right\rangle-\varphi\left(x\right)=\varnothing$. In summary, $\partial\varphi^\star\left(x^*\right)\subset\arg\max_{x\in E}\left\langle x^*,x\right\rangle-\varphi\left(x\right)$.
%%%%%%%%%%%%%%%%%%%%%%%%%%%%%%%%%%%%%%

%%%%%%%%%%%%%%%%%%%%%%%%%%%%%%%%%%%%%%
On the other hand, if $\arg\max_{x\in E}\left\langle x^*,x\right\rangle-\varphi\left(x\right)\neq\varnothing$, then $\forall x\in\arg\max_{x\in E}\left\langle x^*,x\right\rangle-\varphi\left(x\right)$, $\left\langle x^*,x\right\rangle-\varphi\left(x\right)=\sup_{x\in E}\left\langle x^*,x\right\rangle-\varphi\left(x\right)$ holds, that is, the equation $\varphi^\star\left(x^*\right)=\left\langle x^*,x\right\rangle-\varphi\left(x\right)$ holds, thus, $x\in\partial\varphi^\star\left(x^*\right)$;
If $\arg\max_{x\in E}\left\langle x^*,x\right\rangle-\varphi\left(x\right)=\varnothing$, then $\forall x\in E$, $\left\langle x^*,x\right\rangle-\varphi\left(x\right)=\arg\max_{x\in E}\left\langle x^*,x\right\rangle-\varphi\left(x\right)$ does not hold, that is, the equation $\varphi^\star\left(x^*\right)=\left\langle x^*,x\right\rangle-\varphi\left(x\right)$ does not hold, thus $\partial\varphi^\star\left(x^*\right)=\varnothing$. In summary, $\arg\max_{x\in E}\left\langle x^*,x\right\rangle-\varphi\left(x\right)\subset\partial\varphi^\star\left(x^*\right)$.
%%%%%%%%%%%%%%%%%%%%%%%%%%%%%%%%%%%%%%

%%%%%%%%%%%%%%%%%%%%%%%%%%%%%%%%%%%%%%
Therefore, we have $\partial\varphi^\star\left(x^*\right)=\arg\max_{x\in E}\left\langle x^*,x\right\rangle-\varphi\left(x\right)$. 
%%%%%%%%%%%%%%%%%%%%%%%%%%%%%%%%%%%%%%

%%%%%%%%%%%%%%%%%%%%%%%%%%%%%%%%%%%%%%
Similarly, we have $\partial\varphi\left(x\right)=\arg\max_{x^*\in E^*}\left\langle x^*,x\right\rangle-\varphi^\star\left(x^*\right)$. 
\end{proof}
%%%%%%%%%%%%%%%%%%%%%%%%%%%%%%%%%%%%%%

%%%%%%%%%%%%%%%%%%%%%%%%%%%%%%%%%%%%%%
\section{Proof of \cref{affine-convex}}
%%%%%%%%%%%%%%%%%%%%%%%%%%%%%%%%%%%%%%

%%%%%%%%%%%%%%%%%%%%%%%%%%%%%%%%%%%%%%
\begin{proof}
Let $\alpha=\ell+c$, where $\ell\in E^*$ and $c\in\mathbb{R}$. Let $\left(y, y^{\varphi+\alpha}\right)\in\gra\partial\left(\varphi+\alpha\right)$, we have  $y^{\varphi+\alpha}\in\partial\left(\varphi+\alpha\right)y=\partial\varphi\left(y\right)+\ell$, i.e., $\left(y, y^{\varphi+\alpha}-\ell\right)\in\gra\partial\varphi$, then $\forall x\in E$, 
\begin{equation*}
\begin{aligned}
B_{\varphi+\alpha}\left(x, y^{\varphi+\alpha}\right)
&=\left(\varphi+\alpha\right)x+\left(\varphi+\alpha\right)^\star y^{\varphi+\alpha}-\left\langle y^{\varphi+\alpha},x\right\rangle \\
&=\varphi\left(x\right)+\varphi^\star\left(y^{\varphi+\alpha}-\ell\right)-\left\langle y^{\varphi+\alpha}-\ell,x\right\rangle \\
&=B_{\varphi}\left(x, y^{\varphi+\alpha}-\ell\right)
\geqslant\phi\left(\left\lVert x-y\right\rVert\right).
\end{aligned}
\end{equation*}
\end{proof}
%%%%%%%%%%%%%%%%%%%%%%%%%%%%%%%%%%%%%%

%%%%%%%%%%%%%%%%%%%%%%%%%%%%%%%%%%%%%%
\section{Supplementary Proof of \cref{S-I}}
\label{pf:S-I}
%%%%%%%%%%%%%%%%%%%%%%%%%%%%%%%%%%%%%%

%%%%%%%%%%%%%%%%%%%%%%%%%%%%%%%%%%%%%%
The supplementary proof of \cref{S-I} relies on the following lemma. 
%%%%%%%%%%%%%%%%%%%%%%%%%%%%%%%%%%%%%%

%%%%%%%%%%%%%%%%%%%%%%%%%%%%%%%%%%%%%%
\begin{lemma}[Bolzano-Weierstra{\ss}.~Theorem~6.21 of  {\citealp{muscat2014functional}}]
\label{Bolzano-Weierstrass}
In a metric space, a subset $C$ is compact iff every sequence in $C$ has a subsequence that converges in $C$.
\end{lemma}
%%%%%%%%%%%%%%%%%%%%%%%%%%%%%%%%%%%%%%

%%%%%%%%%%%%%%%%%%%%%%%%%%%%%%%%%%%%%%
\begin{proof}
It suffices to prove that, under sufficient conditions that $\partial\psi\left(C\right)=E^*$ or $C$ is compact, $x\in C$ can be guaranteed by the following general expression included in the update rule, 
\begin{equation}
\label{general-rule}
\begin{aligned}
x\in\arg\min_{E}\left(\psi-x^*\right) 
=\partial\psi^\star\left(x^*\right), \quad\forall x^*\in E^*.
\end{aligned}
\end{equation}
\begin{itemize}
\item[(a)] If $\partial\psi\left(C\right)=E^*$, then \cref{general-rule} is equivalent to $x^*\in\partial\psi\left(x\right)$ according to \cref{Symmetry-of-subdifferential}. 
Since $\dom\partial\psi=C$, we have that $x\in C$. \\
\item[(b)] If $C$ is compact, according to \cref{subdifferential-convexity}, $f=\psi-x^*$ is convex and lower semicontinuous on $C$, 
thich implies that $\left\{\left.f^{-1}\left(\alpha, +\infty\right]\,\right| \alpha\in\mathbb{R}\right\}$ is an open cover of $C$, 
then $\exists m\in \mathbb{N}$, such that $\left\{f^{-1}\left(\alpha_i, +\infty\right]\right\}_{i=1}^m$ covers $C$, 
that is, $f\left(C\right)>\min_i\alpha_i$. Let $\beta=\inf f\left(C\right)$, then $\exists y_n\in C$, such that 
\begin{equation*}
\beta\leqslant f\left(y_n\right)<\beta+\frac{1}{n},
\end{equation*}
according to \cref{Bolzano-Weierstrass}, there exists a subsequence $y_{n_k}\rightarrow y_0\in C$, which implies that 
\begin{equation*}
\beta\leftarrow f\left(y_{n_k}\right)\rightarrow f\left(y_0\right),
\end{equation*}
thus $f\left(y_0\right)=\beta$, which leading to $x=y_0\in C$.
\end{itemize}
\end{proof}
%%%%%%%%%%%%%%%%%%%%%%%%%%%%%%%%%%%%%%

%%%%%%%%%%%%%%%%%%%%%%%%%%%%%%%%%%%%%%
\section{Proof of \cref{static-reduce}}
%%%%%%%%%%%%%%%%%%%%%%%%%%%%%%%%%%%%%%

%%%%%%%%%%%%%%%%%%%%%%%%%%%%%%%%%%%%%%
\begin{proof}
Note that 
\begin{equation*}
\begin{aligned}
&\sum_{t=1}^{T}\frac{1}{\eta_t}\left[B_{\psi}\left(z, \widecheck{x}_t^\psi\right)-B_{\psi}\left(z, a^\psi+\frac{\eta_t}{\eta_{t+1}}\left(\widecheck{x}_{t+1}^\psi-a^\psi\right)\right)\right] \\
\leqslant\ &\frac{1}{\eta_1}B_{\psi}\left(z, \widecheck{x}_1^\psi\right)
+\sum_{t=1}^{T}\left[\frac{1}{\eta_{t+1}}B_{\psi}\left(z, \widecheck{x}_{t+1}^\psi\right)-\frac{1}{\eta_{t}}B_{\psi}\left(z,\, a^\psi+\frac{\eta_t}{\eta_{t+1}}\left(\widecheck{x}_{t+1}^\psi-a^\psi\right)\right)\right] \\
\leqslant\ &\frac{1}{\eta_1}B_{\psi}\left(z, \widecheck{x}_1^\psi\right)+\sum_{t=1}^{T}\left(\frac{1}{\eta_{t+1}}-\frac{1}{\eta_{t}}\right)B_{\psi}\left(z, a^\psi\right)
=\frac{1}{\eta_{T+1}}B_{\psi}\left(z, a^\psi\right),
\end{aligned}
\end{equation*}
where the second ``$\leqslant$'' follows from the convexity of $B_{\psi}\left(z,\cdot\right)$. Let $\mu_t=\frac{\eta_{t+1}}{\eta_t}\leqslant 1$, 
\begin{equation*}
\begin{aligned}
\frac{1}{\eta_{t+1}}B_{\psi}\left(z, \widecheck{x}_{t+1}^\psi\right)-\frac{1}{\eta_{t}}B_{\psi}\left(z,\, a^\psi+\frac{\eta_t}{\eta_{t+1}}\left(\widecheck{x}_{t+1}^\psi-a^\psi\right)\right)
\leqslant\left(\frac{1}{\eta_{t+1}}-\frac{1}{\eta_{t}}\right)B_{\psi}\left(z, a^\psi\right)
\end{aligned}
\end{equation*}
is equivalent to 
\begin{equation*}
\begin{aligned}
B_{\psi}\left(z, \widecheck{x}_{t+1}^\psi\right)\leqslant
\mu_t B_{\psi}\left(z,\, a^\psi+\frac{1}{\mu_t}\left(\widecheck{x}_{t+1}^\psi-a^\psi\right)\right)
+\left(1-\mu_t\right)B_{\psi}\left(z, a^\psi\right).
\end{aligned}
\end{equation*}
%%%%%%%%%%%%%%%%%%%%%%%%%%%%%%%%%%%%%%

%%%%%%%%%%%%%%%%%%%%%%%%%%%%%%%%%%%%%%
For the second half of \cref{static-reduce}, 
\begin{equation*}
\begin{aligned}
&\sum_{t=1}^{T}\frac{1}{\theta_t}\left[B_{\psi}\left(z_t, \widecheck{x}_{t}^\psi\right)-B_{\psi}\left(z_t, \widecheck{x}_{t+1}^\psi\right)\right] \\
\leqslant\ &\frac{1}{\theta_1}B_{\psi}\left(z_1, \widecheck{x}_{1}^\psi\right)-\frac{1}{\theta_T}B_{\psi}\left(z_T, \widecheck{x}_{T+1}^\psi\right)
+\sum_{t=2}^{T}\frac{1}{\theta_t}\left[B_{\psi}\left(z_t, \widecheck{x}_{t}^\psi\right)-B_{\psi}\left(z_{t-1}, \widecheck{x}_{t}^\psi\right)\right] \\
\leqslant\ &
\frac{1}{\theta_{1}}B_{\psi}\left(z, a^\psi\right)+\sum_{t=2}^{T}\frac{1}{\theta_t}\left\langle \partial\psi\left(z_t\right)-\widecheck{x}_{t}^\psi, z_t-z_{t-1}\right\rangle,
\end{aligned}
\end{equation*}
where the last ``$\leqslant$'' follows from the convexity of $\psi$. 
\end{proof}
%%%%%%%%%%%%%%%%%%%%%%%%%%%%%%%%%%%%%%

%%%%%%%%%%%%%%%%%%%%%%%%%%%%%%%%%%%%%%
\section{Proof of \cref{auxiliary}}
%%%%%%%%%%%%%%%%%%%%%%%%%%%%%%%%%%%%%%

%%%%%%%%%%%%%%%%%%%%%%%%%%%%%%%%%%%%%%
This is an extension of Appendix~B of \cite{flaspohler2021online}. 
%%%%%%%%%%%%%%%%%%%%%%%%%%%%%%%%%%%%%%

%%%%%%%%%%%%%%%%%%%%%%%%%%%%%%%%%%%%%%
\begin{proof}
The instantaneous dynamic regret can be decomposed as the following form, 
\begin{equation*}
\begin{aligned}
\varphi_t\left(x_t\right)-\varphi_t\left(z_t\right)
\leqslant\left\langle x_t^*,x_t-z_t\right\rangle
=\underbrace{\left\langle x_t^*,x_t-y_t\right\rangle}_{\text{drift}}+\underbrace{\left\langle x_t^*,y_t-z_t\right\rangle}_{\text{auxiliary}}, \quad x_t^*\in\partial\varphi_t\left(x_t\right).
\end{aligned}
\end{equation*}
The regret bound for the auxiliary term is simply replacing $x_{t}$, $\widecheck{x}_{t}^\psi$ and $\widetilde{x}_{t}^\psi$ with $y_{t}$, $\widecheck{y}_{t}^\psi$ and $\widetilde{y}_{t}^\psi$ respectively. 
For the drift term, note that $f_t\left(y\right)\coloneqq\left\langle \eta_t\sum_{i=1}^{t-1}\theta_i x_i^* ,\, y\right\rangle+B_{\psi}\left(y, x_1^\psi\right)$ is $\phi$-convex according to \cref{affine-convex}, then 
\begin{equation*}
\begin{aligned}
\left(f_t\left(x_t\right)+\left\langle \eta_t\theta_t \widehat{y}_t^* , x_t\right\rangle\right)-\left(f_t\left(y_t\right)+\left\langle \eta_t\theta_t \widehat{y}_t^* , y_t\right\rangle\right)
&\geqslant\phi\left(\left\lVert x_t-y_t\right\rVert\right), \\
\left(f_t\left(y_t\right)+\left\langle \eta_t\theta_t \widehat{x}_t^* , y_t\right\rangle\right)-\left(f_t\left(x_t\right)+\left\langle \eta_t\theta_t \widehat{x}_t^* , x_t\right\rangle\right)
&\geqslant\phi\left(\left\lVert x_t-y_t\right\rVert\right).
\end{aligned}
\end{equation*}
Adding the above two inequalities and utilizing the Fenchel-Young inequality, we have 
\begin{equation*}
\phi\left(\left\lVert x_t-y_t\right\rVert\right)
\leqslant\phi^\star\left(\eta_t\theta_t\left\lVert \widehat{x}_t^*-\widehat{y}_t^*\right\rVert\right). 
\end{equation*}
Then $\forall\gamma>0$, 
\begin{equation*}
\begin{aligned}
\left\langle x_t^*,x_t-y_t\right\rangle
\leqslant \frac{1}{\gamma}\phi^\star\left(\gamma\left\lVert x_t^*\right\rVert\right)+\frac{1}{\gamma}\phi\left(\left\lVert x_t-y_t\right\rVert\right)
\leqslant \frac{1}{\gamma}\phi^\star\left(\gamma\left\lVert x_t^*\right\rVert\right)+\frac{1}{\gamma}\phi^\star\left(\eta_t\theta_t\left\lVert \widehat{x}_t^*-\widehat{y}_t^*\right\rVert\right). 
\end{aligned}
\end{equation*}
Thus, the term $\phi^\star\left(\xi\left\lVert x_t^*-\widehat{x}_t^*\right\rVert\right)$ in regret upper bounds is replaced by $\mathit{\Phi}_{\xi}\left(x_t^*,\widehat{x}_t^*\right)$, where 
\begin{equation*}
\begin{aligned}
\mathit{\Phi}_{\xi}\left(x^*,\widehat{x}^*\right)=\phi^\star\left(\xi\left\lVert x^*-\widehat{y}^*\right\rVert\right)+\inf_{\gamma>0}\left(\frac{1}{\gamma}\phi^\star\left(\gamma\xi\left\lVert x^*\right\rVert\right)+\frac{1}{\gamma}\phi^\star\left(\xi\left\lVert \widehat{x}^*-\widehat{y}^*\right\rVert\right)\right).
\end{aligned}
\end{equation*}
%%%%%%%%%%%%%%%%%%%%%%%%%%%%%%%%%%%%%%

%%%%%%%%%%%%%%%%%%%%%%%%%%%%%%%%%%%%%%
If $\phi=Q_\rho$, then 
\begin{equation*}
\begin{aligned}
\mathit{\Phi}_{\xi}\left(x^*,\widehat{x}^*\right)
=Q_\rho^\star\left(\xi\left\lVert x^*-\widehat{y}^*\right\rVert\right)+\inf_{\gamma>0}\left(\frac{1}{\gamma}Q_\rho^\star\left(\gamma\xi\left\lVert x^*\right\rVert\right)+\frac{1}{\gamma}Q_\rho^\star\left(\xi\left\lVert \widehat{x}^*-\widehat{y}^*\right\rVert\right)\right).
\end{aligned}
\end{equation*}
Let $f\left(\gamma\right)=\frac{1}{\gamma}Q_\rho^\star\left(\gamma\xi\left\lVert x^*\right\rVert\right)+\frac{1}{\gamma}Q_\rho^\star\left(\xi\left\lVert \widehat{x}^*-\widehat{y}^*\right\rVert\right)$. 
If $\xi\left\lVert \widehat{x}^*-\widehat{y}^*\right\rVert>\rho$, then 
\begin{equation*}
\begin{aligned}
\inf_{\gamma>0}f\left(\gamma\right)=\min\left\{\inf_{\gamma\xi\left\lVert x^*\right\rVert>\rho}f\left(\gamma\right),\,\, \inf_{\gamma>0,\,\gamma\xi\left\lVert x^*\right\rVert\leqslant\rho}f\left(\gamma\right)\right\}=\rho\xi\left\lVert x^*\right\rVert
\end{aligned}
\end{equation*}
since 
\begin{equation*}
\begin{aligned}
\inf_{\gamma\xi\left\lVert x^*\right\rVert>\rho}f\left(\gamma\right)
&=\rho\xi\left\lVert x^*\right\rVert+\inf_{\gamma\xi\left\lVert x^*\right\rVert>\rho}\frac{\rho}{\gamma}\left(\xi\left\lVert \widehat{x}^*-\widehat{y}^*\right\rVert-\rho\right)
=\rho\xi\left\lVert x^*\right\rVert, \\
\inf_{\gamma>0,\,\gamma\xi\left\lVert x^*\right\rVert\leqslant\rho}f\left(\gamma\right)
&=\inf_{\gamma>0,\,\gamma\xi\left\lVert x^*\right\rVert\leqslant\rho}\left(\frac{\gamma}{2}\left(\xi\left\lVert x^*\right\rVert\right)^2+\frac{\rho}{2\gamma}\left(2\xi\left\lVert \widehat{x}^*-\widehat{y}^*\right\rVert-\rho\right)\right) \\
&\geqslant\inf_{\gamma>0,\,\gamma\xi\left\lVert x^*\right\rVert\leqslant\rho}\left(\frac{\gamma}{2}\left(\xi\left\lVert x^*\right\rVert\right)^2+\frac{\rho^2}{2\gamma}\right)
\geqslant\rho\xi\left\lVert x^*\right\rVert. 
\end{aligned}
\end{equation*}
If $\xi\left\lVert \widehat{x}^*-\widehat{y}^*\right\rVert\leqslant\rho$, then 
\begin{equation*}
\begin{aligned}
\inf_{\gamma>0}f\left(\gamma\right)=\min\left\{ \inf_{\gamma>0,\,\gamma\xi\left\lVert x^*\right\rVert\leqslant\rho}f\left(\gamma\right),\,\,\inf_{\gamma\xi\left\lVert x^*\right\rVert>\rho}f\left(\gamma\right)\right\}=\xi^2\left\lVert x^*\right\rVert\left\lVert \widehat{x}^*-\widehat{y}^*\right\rVert
\end{aligned}
\end{equation*}
since 
\begin{equation*}
\begin{aligned}
\inf_{\gamma>0,\,\gamma\xi\left\lVert x^*\right\rVert\leqslant\rho}f\left(\gamma\right)
&=\inf_{\gamma>0,\,\gamma\xi\left\lVert x^*\right\rVert\leqslant\rho}\left(\frac{\gamma}{2}\left(\xi\left\lVert x^*\right\rVert\right)^2+\frac{1}{2\gamma}\left(\xi\left\lVert \widehat{x}^*-\widehat{y}^*\right\rVert\right)^2\right)
=\xi^2\left\lVert x^*\right\rVert\left\lVert \widehat{x}^*-\widehat{y}^*\right\rVert, \\
\inf_{\gamma\xi\left\lVert x^*\right\rVert>\rho}f\left(\gamma\right)
&=\rho\xi\left\lVert x^*\right\rVert + \inf_{\gamma\xi\left\lVert x^*\right\rVert>\rho}\frac{1}{2\gamma}\left(\xi^2\left\lVert \widehat{x}^*-\widehat{y}^*\right\rVert^2-\rho^2\right) \\
%&=\rho\xi\left\lVert x^*\right\rVert + \frac{\xi\left\lVert x^*\right\rVert}{2\rho}\left(\eta_t^2\theta_t^2\left\lVert \widehat{x}^*-\widehat{y}^*\right\rVert^2-\rho^2\right) \\
&=\frac{\rho}{2}\xi\left\lVert x^*\right\rVert + \frac{\xi\left\lVert x^*\right\rVert}{2\rho}\xi^2\left\lVert \widehat{x}^*-\widehat{y}^*\right\rVert^2
\geqslant \xi^2\left\lVert x^*\right\rVert\left\lVert \widehat{x}^*-\widehat{y}^*\right\rVert.
\end{aligned}
\end{equation*}
In summary, 
\begin{equation*}
\begin{aligned}
\inf_{\gamma>0}\left(\frac{1}{\gamma}Q_\rho^\star\left(\gamma\xi\left\lVert x^*\right\rVert\right)+\frac{1}{\gamma}Q_\rho^\star\left(\xi\left\lVert \widehat{x}^*-\widehat{y}^*\right\rVert\right)\right)
%=\inf_{\gamma>0}f\left(\gamma\right)
=\xi\left\lVert x^*\right\rVert\min\left\{\xi\left\lVert \widehat{x}^*-\widehat{y}^*\right\rVert, \rho\right\}.
\end{aligned}
\end{equation*}
%%%%%%%%%%%%%%%%%%%%%%%%%%%%%%%%%%%%%%

%%%%%%%%%%%%%%%%%%%%%%%%%%%%%%%%%%%%%%
If $\widehat{y}^*=\lambda\widehat{x}^*+\left(1-\lambda\right)x^*$, $\lambda=\min\left\{\frac{\left\lVert x^*\right\rVert}{\left\lVert x^*-\widehat{x}^*\right\rVert},\,1\right\}$, then  
\begin{equation*}
\begin{aligned}
&\mathit{\Phi}_{\xi}\left(x^*,\widehat{x}^*\right) \\
=\ &Q_\rho^\star\left(\xi\left\lVert x^*-\widehat{y}^*\right\rVert\right)+\xi\left\lVert x^*\right\rVert\min\left\{\xi\left\lVert \widehat{x}^*-\widehat{y}^*\right\rVert, \rho\right\} \\
=\ &Q_\rho^\star\big(\xi\min\left\{\left\lVert x^*-\widehat{x}^*\right\rVert, \left\lVert x^*\right\rVert\right\}\big)+\xi\left\lVert x^*\right\rVert\min\left\{\xi\left(\left\lVert x^*-\widehat{x}^*\right\rVert-\left\lVert x^*\right\rVert\right)_+, \rho\right\}\\
\leqslant\ & Q_\rho^\star\big(\xi\min\left\{\left\lVert x^*-\widehat{x}^*\right\rVert, \left\lVert x^*\right\rVert\right\}\big)
+\xi^2\left\lVert x^*\right\rVert\left(\left\lVert x^*-\widehat{x}^*\right\rVert-\left\lVert x^*\right\rVert\right)_+ \\
=\ &\left\{\begin{array}{ll}
\displaystyle\frac{1}{2}\xi^2\left\lVert x^*-\widehat{x}^*\right\rVert^2-\frac{1}{2}\xi^2\left(\left\lVert x^*-\widehat{x}^*\right\rVert-\left\lVert x^*\right\rVert\right)^2-\frac{1}{2}\left(\xi\left\lVert x^*\right\rVert-\rho\right)_+^2,&\left\lVert x^*-\widehat{x}^*\right\rVert>\left\lVert x^*\right\rVert \\[8pt]
\displaystyle\frac{1}{2}\xi^2\left\lVert x^*-\widehat{x}^*\right\rVert^2-\frac{1}{2}\left(\xi\left\lVert x^*-\widehat{x}^*\right\rVert-\rho\right)_+^2,&\left\lVert x^*-\widehat{x}^*\right\rVert\leqslant\left\lVert x^*\right\rVert
\end{array}\right. \\
=\ &\xi^2 Q_{\left\lVert x^*\right\rVert}^\star\left(\left\lVert x^*-\widehat{x}^*\right\rVert\right)-\frac{1}{2}\big(\xi\min\left\{\left\lVert x^*-\widehat{x}^*\right\rVert, \left\lVert x^*\right\rVert\right\}-\rho\big)_+^2 .
\end{aligned}
\end{equation*}
\end{proof}
%%%%%%%%%%%%%%%%%%%%%%%%%%%%%%%%%%%%%%

%%%%%%%%%%%%%%%%%%%%%%%%%%%%%%%%%%%%%%
\section{Proof of \cref{Negative-Entropy}}
%%%%%%%%%%%%%%%%%%%%%%%%%%%%%%%%%%%%%%

%%%%%%%%%%%%%%%%%%%%%%%%%%%%%%%%%%%%%%
The proof of \cref{Negative-Entropy} relies on the following lemmas. 
%%%%%%%%%%%%%%%%%%%%%%%%%%%%%%%%%%%%%%

%%%%%%%%%%%%%%%%%%%%%%%%%%%%%%%%%%%%%%
\begin{lemma}[Heine-Borel.~Theorem~5.5 of {\citealp{gamelin1999introduction}}]
\label{Heine-Borel}
Let $S$ be a subset of metric space $\mathbb{R}^n$, then $S$ is compact iff $S$ is closed and bounded.
\end{lemma}
%%%%%%%%%%%%%%%%%%%%%%%%%%%%%%%%%%%%%%

%%%%%%%%%%%%%%%%%%%%%%%%%%%%%%%%%%%%%%
\begin{lemma}[Example 2.5 of {\citealp{shwartz2012online}}]
\label{Entropy-strongly-convex}
$\psi\left(w\right)=\left\langle w, \ln w\right\rangle+\chi_{\bigtriangleup^{n}}\left(w\right)$ is $1$-strongly-convex w.r.t $\left\lVert\cdot\right\rVert_1$ over the probability simplex $\bigtriangleup^{n}$.
\end{lemma}
%%%%%%%%%%%%%%%%%%%%%%%%%%%%%%%%%%%%%%

%%%%%%%%%%%%%%%%%%%%%%%%%%%%%%%%%%%%%%
\begin{lemma}[Table~2.1 of {\citealp{shwartz2012online}}]
\label{Entropy-conjugation}
$\psi\left(w\right)=\left\langle w, \ln w\right\rangle+\chi_{\bigtriangleup^{n}}\left(w\right)$ and 
$\psi^\star\left(w^*\right)=\ln\left\langle\mathbf{1}, \mathrm{e}^{w^*}\right\rangle$ are a pair of Fenchel conjugate functions, where $\mathbf{1}$ represents the all-ones vector in $\mathbb{R}^{n+1}$. 
\end{lemma}
%%%%%%%%%%%%%%%%%%%%%%%%%%%%%%%%%%%%%%

%%%%%%%%%%%%%%%%%%%%%%%%%%%%%%%%%%%%%%
\begin{proof}
Note that $\bigtriangleup^{n}$ is a bounded closed subset of $\mathbb{R}^{n+1}$ with metric $\left\lVert\cdot\right\rVert_1$. According to \cref{Heine-Borel}, $\bigtriangleup^{n}$ is a compact subset. 
%%%%%%%%%%%%%%%%%%%%%%%%%%%%%%%%%%%%%%

%%%%%%%%%%%%%%%%%%%%%%%%%%%%%%%%%%%%%%
\cref{Entropy-strongly-convex} is equivalent to 
\begin{equation*}
\begin{aligned}
\psi\left(w\right)-\psi\left(u\right)\geqslant\left\langle u^\psi, w-u\right\rangle+\frac{1}{2}\left\lVert w-u\right\rVert_1^2, \quad \forall u, w\in\bigtriangleup^{n},\quad\forall u^\psi\in\partial \psi\left(u\right). 
\end{aligned}
\end{equation*}
Remove the restriction on $w$ by adding $\chi_{\bigtriangleup^{n}}\left(w\right)$ on both sides of the above inequality, and note that $\left\lVert w-u\right\rVert_1\leqslant\left\lVert w\right\rVert_1+\left\lVert u\right\rVert_1=2$ for all  $w\in\bigtriangleup^{n}$, we obtain that 
\begin{equation*}
\begin{aligned}
B_{\psi}\left(w, u^\psi\right)\geqslant\frac{1}{2}\left\lVert w-u\right\rVert_1^2+\chi_{\left[-2, 2\right]}\left(\left\lVert w-u\right\rVert_1\right),\quad\forall u\in\bigtriangleup^{n}, \quad\forall u^\psi\in\partial \psi\left(u\right), 
\end{aligned}
\end{equation*}
which implies that $\psi$ is $Q_2$-convex. 
%%%%%%%%%%%%%%%%%%%%%%%%%%%%%%%%%%%%%%

%%%%%%%%%%%%%%%%%%%%%%%%%%%%%%%%%%%%%%
According to \cref{Entropy-conjugation}, we have that 
\begin{equation*}
\begin{aligned}
\partial\psi\left(w\right)&\ni\boldsymbol{1}+\ln w, \quad\forall  w\in\bigtriangleup^{n},\quad w>0, \\
\partial\psi^\star\left(w^*\right)&=\mathscr{N}\mathrm{e}^{w^*}.
\end{aligned}
\end{equation*}
Choose $\partial\varphi=\boldsymbol{1}+\ln$, then the strategy $S$ is instantiated as 
\begin{equation*}
w_{t}=\mathscr{N}\left(a\circ\mathrm{e}^{-\eta_t\left(\sum_{i=1}^{t-1}\theta_i w_i^*+\theta_t\widehat{w}_t^*\right)}\right). 
\end{equation*}
To complete the proof, it suffices to replace $w_i^*$ and $\widehat{w}_t^*$ with $\ell_i$ and $\widehat{\ell}_t$ respectively. 
\end{proof}
%%%%%%%%%%%%%%%%%%%%%%%%%%%%%%%%%%%%%%

%%%%%%%%%%%%%%%%%%%%%%%%%%%%%%%%%%%%%%
\section{Proof of \cref{ONES}}
%%%%%%%%%%%%%%%%%%%%%%%%%%%%%%%%%%%%%%

%%%%%%%%%%%%%%%%%%%%%%%%%%%%%%%%%%%%%%
\begin{proof}
$\forall u, w\in\bigtriangleup^{n}$, $w^\psi=\boldsymbol{1}+\ln w$ since $\partial\varphi=\boldsymbol{1}+\ln$. 
By \cref{Entropy-conjugation}, 
\begin{equation*}
\begin{aligned}
B_{\psi}\left(u, w^\psi\right)
=\left\langle u, \ln u\right\rangle+\ln\left\langle\mathbf{1}, \mathrm{e}^{w^\psi}\right\rangle-\left\langle u, w^\psi\right\rangle
=\left\langle u, \ln \frac{u}{w}\right\rangle. 
\end{aligned}
\end{equation*}
Note that $\bigtriangleup^{n}$ is compact and $\psi$ is $Q_2$-convex. 
To complete the proof, we just apply \cref{S-static}. 
\end{proof}
%%%%%%%%%%%%%%%%%%%%%%%%%%%%%%%%%%%%%%

%%%%%%%%%%%%%%%%%%%%%%%%%%%%%%%%%%%%%%
\section{Proof of \cref{squared-norm}}
%%%%%%%%%%%%%%%%%%%%%%%%%%%%%%%%%%%%%%

%%%%%%%%%%%%%%%%%%%%%%%%%%%%%%%%%%%%%%
The proof of \cref{squared-norm} relies on the following lemma. 
%%%%%%%%%%%%%%%%%%%%%%%%%%%%%%%%%%%%%%

%%%%%%%%%%%%%%%%%%%%%%%%%%%%%%%%%%%%%%
\begin{lemma}[Theorem~5.2 of {\citealp{brezis2010functional}}]
\label{lm:projection}
Let $C\neq\varnothing$ be a closed convex subset of Hilbert space $E$. $\forall x\in E$, $\exists ! x_0\in C$, such that $\left\lVert x-x_0\right\rVert=\inf\left\lVert x-C\right\rVert$. $x_0$ is called the projection of $x$ onto $C$ and is denoted by $x_0=P_C\left(x\right)$. 
Moreover, $\left\langle x-x_0, C-x_0\right\rangle\leqslant 0$. 
\end{lemma}
%%%%%%%%%%%%%%%%%%%%%%%%%%%%%%%%%%%%%%

%%%%%%%%%%%%%%%%%%%%%%%%%%%%%%%%%%%%%%
\begin{proof}
According to \cref{lm:projection}, $\forall x\in E$, $\exists ! x_0=P_C\left(x\right)\in C$, such that 
\begin{equation*}
\frac{1}{2}\left\lVert x-y\right\rVert^2+\chi_C\left(y\right)\geqslant\frac{1}{2}\left\lVert x-x_0\right\rVert^2,\quad \forall y\in E,
\end{equation*}
rearrange the above formula, we have 
\begin{equation*}
\frac{1}{2}\left\lVert y\right\rVert^2+\chi_C\left(y\right)-\frac{1}{2}\left\lVert x_0\right\rVert^2\geqslant\left\langle x, y-x_0\right\rangle,\quad \forall y\in E,
\end{equation*}
thus, $x\in\partial\psi\left(x_0\right)$, which shows that $\partial\psi\left(C\right)=E$. 
%%%%%%%%%%%%%%%%%%%%%%%%%%%%%%%%%%%%%%

%%%%%%%%%%%%%%%%%%%%%%%%%%%%%%%%%%%%%%
Moreover, we have a stronger result that $P_C=\partial\psi^\star$, or equivalently, $x_0$ is the unique image of $x$ over $\partial\psi^\star$. 
Suppose by contradiction that there is some $x_0'\neq x_0$ such that $x_0'\in\partial\psi^\star\left(x\right)$ (and obviously $x_0'\in C$), by \cref{Symmetry-of-subdifferential}, it follows that $x\in\partial\psi\left(x_0'\right)$, that is, 
\begin{equation*}
\frac{1}{2}\left\lVert y\right\rVert^2+\chi_C\left(y\right)-\frac{1}{2}\left\lVert x_0'\right\rVert^2\geqslant\left\langle x, y-x_0'\right\rangle,\quad \forall y\in E,
\end{equation*}
then $\forall y\in C$, $\left\lVert x-y\right\rVert\geqslant\left\lVert x-x_0'\right\rVert$, which implies that $x_0'=P_C\left(x\right)=x_0$, which is a contradiction. 
%%%%%%%%%%%%%%%%%%%%%%%%%%%%%%%%%%%%%%

%%%%%%%%%%%%%%%%%%%%%%%%%%%%%%%%%%%%%%
Now we prove that $\psi$ is $Q_\rho$-convex. 
Note that 
\begin{equation*}
\begin{aligned}
\psi^\star\left(x^*\right)
&=\sup_{x\in E}\left\langle x^*, x\right\rangle-\frac{1}{2}\left\lVert x\right\rVert^2-\chi_C\left(x\right)
=\frac{1}{2}\left\lVert x^*\right\rVert^2+\sup_{x\in C}\left\langle x^*, x\right\rangle-\frac{1}{2}\left\lVert x\right\rVert^2-\frac{1}{2}\left\lVert x^*\right\rVert^2 \\
&=\frac{1}{2}\left\lVert x^*\right\rVert^2-\frac{1}{2}\inf_{x\in C}\left\lVert x^*-x\right\rVert^2
=\frac{1}{2}\left\lVert x^*\right\rVert^2-\frac{1}{2}\left\lVert x^*-P_C\left(x^*\right)\right\rVert^2. 
\end{aligned}
\end{equation*}
$\forall z\in C$, $\forall z^\psi\in\partial\psi\left(z\right)$, we have that $z=P_C\left(z^\psi\right)$, and according to the cosine rule, 
\begin{equation*}
\begin{aligned}
\left\lVert z^\psi-y\right\rVert^2=\left\lVert z^\psi-z\right\rVert^2+\left\lVert y-z\right\rVert^2-2\left\langle z^\psi-z, y-z\right\rangle. 
\end{aligned}
\end{equation*}
Let $y\in C$, then by \cref{lm:projection}, 
\begin{equation*}
\begin{aligned}
\left\lVert z^\psi-y\right\rVert^2\geqslant\left\lVert z^\psi-z\right\rVert^2+\left\lVert y-z\right\rVert^2, 
\end{aligned}
\end{equation*}
rearrange the above inequality, and note that $\chi_C\left(y\right)\geqslant\chi_\rho\left(\left\lVert y-z\right\rVert\right)$, 
\begin{equation*}
\begin{aligned}
\frac{1}{2}\left\lVert y\right\rVert^2+\chi_C\left(y\right)-\frac{1}{2}\left\lVert z\right\rVert^2 - \left\langle z^\psi, y-z\right\rangle
&\geqslant\frac{1}{2}\left\lVert y-z\right\rVert^2+\chi_\rho\left(\left\lVert y-z\right\rVert\right),
\end{aligned}
\end{equation*}
where 
\begin{equation*}
\begin{aligned}
-\frac{1}{2}\left\lVert z\right\rVert^2+\left\langle z^\psi, z\right\rangle=-\frac{1}{2}\left\lVert z^\psi\right\rVert^2+\frac{1}{2}\left\lVert z^\psi-P_C\left(z^\psi\right)\right\rVert^2=\psi^\star\left(z^\psi\right). 
\end{aligned}
\end{equation*}
Therefore, $B_\psi\left(y, z^\psi\right)\geqslant Q_\rho\left(\left\lVert y-z\right\rVert\right)$. 
%%%%%%%%%%%%%%%%%%%%%%%%%%%%%%%%%%%%%%

%%%%%%%%%%%%%%%%%%%%%%%%%%%%%%%%%%%%%%
Finally, choose $\partial\varphi$ as the identity map, since $\forall x\in C$, $x\in\partial\varphi\left(x\right)$, and according to $P_C=\partial\psi^\star$ and the uniqueness of the projection, $\si$ and $\sii$ are respectively instantiated into the following forms, 
\begin{equation*}
\begin{aligned}
\widetilde{x}_t&=P_C\left(a-\eta_t\sum_{i=1}^{t-1}\theta_i x_i^*\right), \\
x_t&=P_C\left(\widetilde{x}_t -\eta_t\theta_t \widehat{x}_t^*\right),
\end{aligned}
\end{equation*}
and
\begin{equation*}
\begin{aligned}
\widetilde{x}_t&=P_C\left(\widetilde{x}_{t-1}-\theta_{t-1} x_{t-1}^*\right), \\
x_t&=P_C\left(\widetilde{x}_t -\theta_t \widehat{x}_t^*\right).
\end{aligned}
\end{equation*}
\end{proof}
%%%%%%%%%%%%%%%%%%%%%%%%%%%%%%%%%%%%%%

%%%%%%%%%%%%%%%%%%%%%%%%%%%%%%%%%%%%%%
\section{Proof of \cref{OLP}}
\label{pf:OLP}
%%%%%%%%%%%%%%%%%%%%%%%%%%%%%%%%%%%%%%

%%%%%%%%%%%%%%%%%%%%%%%%%%%%%%%%%%%%%%
\begin{proof}
Since $\partial\varphi$ is the identity map, $\widetilde{x}_t^\psi=\widetilde{x}_t=P_C\left(\widecheck{x}_t^\psi\right)$, $\widecheck{x}_t^\psi=a-\eta_t\sum_{i=1}^{t-1}\theta_i x_i^*$. 
Note that 
\begin{equation*}
\begin{aligned}
B_\psi\left(z,x\right)=\frac{1}{2}\left\lVert z-x\right\rVert^2, \quad \forall x, z\in C. 
\end{aligned}
\end{equation*}
According to \cref{lm:projection}, 
\begin{equation*}
\begin{aligned}
B_{\psi}\left(X_{t+1}, \widetilde{x}_t^\psi\right)-B_{\psi}\left(X_{t+1}, \widecheck{x}_t^\psi\right)&\leqslant\left\langle \widetilde{x}_t^\psi-\widecheck{x}_t^\psi,\widetilde{x}_t-X_{t+1}\right\rangle \\
&=\left\langle P_C\left(\widecheck{x}_t^\psi\right)-\widecheck{x}_t^\psi,P_C\left(\widecheck{x}_t^\psi\right)-X_{t+1}\right\rangle\leqslant 0. 
\end{aligned}
\end{equation*}
To complete the proof, it suffices to substitute the above settings together with \cref{static-reduce} into \cref{S-I}. 
\end{proof}
%%%%%%%%%%%%%%%%%%%%%%%%%%%%%%%%%%%%%%

%%%%%%%%%%%%%%%%%%%%%%%%%%%%%%%%%%%%%%
\section{Basis of Monotone Operators}
\label{BoMO}
%%%%%%%%%%%%%%%%%%%%%%%%%%%%%%%%%%%%%%

%%%%%%%%%%%%%%%%%%%%%%%%%%%%%%%%%%%%%%
All definitions and lemmas in this appendix are compiled from \cite{romano1993potential}, with some minor modifications.
%%%%%%%%%%%%%%%%%%%%%%%%%%%%%%%%%%%%%%

%%%%%%%%%%%%%%%%%%%%%%%%%%%%%%%%%%%%%%
\begin{definition}[Monotone and Cyclic Monotone]
\label{def:Monotone}
Let $C\subset E$, $C\neq\varnothing$, $M\colon\! C\rightrightarrows E^*$, $n\in\mathbb{N}$ and $n\geqslant 2$. $M$ is $n$-cyclically monotone if $\forall\left(x_i,x_i^*\right)\in\gra M$, $i=1, 2, \cdots, n$, 
\begin{equation*}
\sum_{i=1}^n \left\langle x_i^*, x_{i+1}-x_i\right\rangle\leqslant 0,\quad x_{n+1}\equiv x_1.
\end{equation*}
$M$ is monotone iff $M$ is $2$-cyclically monotone; $M$ is  cyclically monotone iff $\forall n\geqslant 2$, $M$ is $n$-cyclically monotone. 
\end{definition}
%%%%%%%%%%%%%%%%%%%%%%%%%%%%%%%%%%%%%%

%%%%%%%%%%%%%%%%%%%%%%%%%%%%%%%%%%%%%%
\begin{remark}
\label{Monotone-equivalent}
The definition of $n$-cyclically monotone  is equivalent to the following statement, $\forall\left(x_i,x_i^*\right)\in\gra M$, $i=1, 2, \cdots, n$, 
\begin{equation*}
\sum_{i=1}^n \left\langle x_i^*, x_{i}-x_{i-1}\right\rangle\geqslant 0,\quad x_{0}\equiv x_n,
\end{equation*}
which is to rearrange the subscripts in original definition in reverse order. 
\end{remark}
%%%%%%%%%%%%%%%%%%%%%%%%%%%%%%%%%%%%%%

%%%%%%%%%%%%%%%%%%%%%%%%%%%%%%%%%%%%%%
\begin{definition}[Integral]
\label{def:Integrals}
$C\neq\varnothing$ is a subset of $E$, $M\colon\! C\rightrightarrows E^*$ is monotone, $\overrightarrow{x y}\colon\!t\in\left[0, 1\right]\mapsto x+t\left(y-x\right)\in C$. The integral of $M$ along $\overrightarrow{x y}$ is defined as 
\begin{equation*}
\int_{\overrightarrow{x y}}\left\langle M \alpha,\,d\alpha\right\rangle\coloneqq\int_0^1 \left\langle M\left(x+t\left(y-x\right)\right), y-x\right\rangle\,d t,\quad\text{and}\quad\int_{\overrightarrow{x y}+\overrightarrow{y z}}=\int_{\overrightarrow{x y}}+\int_{\overrightarrow{y z}}.
\end{equation*}
\end{definition}
%%%%%%%%%%%%%%%%%%%%%%%%%%%%%%%%%%%%%%

%%%%%%%%%%%%%%%%%%%%%%%%%%%%%%%%%%%%%%
\begin{remark}
The integral is well-defined since $m\left(t\right)=\left\langle M\left(x+t\left(y-x\right)\right), y-x\right\rangle$ is a non-decreasing function on $\left[0, 1\right]$ and every non-decreasing function on $\left[0, 1\right]$ is integrable \citep[Theorem~33.1 of][]{ross2013elementary}.
\end{remark}
%%%%%%%%%%%%%%%%%%%%%%%%%%%%%%%%%%%%%%

%%%%%%%%%%%%%%%%%%%%%%%%%%%%%%%%%%%%%%
\begin{lemma}%[Theorem~3.12 of {\citealp{romano1993potential}}]
\label{integral-sup}
Let $M\colon\! E\rightrightarrows E^*$ be a monotone operator, and let $\mathit{\Gamma}\subset\dom M$ be a certain finite-length oriented polyline from $x_0$ to $x$. 
Then 
\begin{equation*}
\begin{aligned}
\int_{\mathit{\Gamma}}\left\langle M\alpha, \,d \alpha\right\rangle
=\sup_{n,x_i,x_i^*}\sum_{i=0}^{n}\left\langle x_{i}^*, x_{i+1}-x_{i}\right\rangle, \quad x_{n+1}\equiv x,
\end{aligned}
\end{equation*}
where the supremum operator traverses all possible decompositions of $\mathit{\Gamma}=\overrightarrow{x_0x_1}+\overrightarrow{x_1x_2}+\cdots+\overrightarrow{x_{n}x}$, $n\in\mathbb{N}+$. 
\end{lemma}
%%%%%%%%%%%%%%%%%%%%%%%%%%%%%%%%%%%%%%

%%%%%%%%%%%%%%%%%%%%%%%%%%%%%%%%%%%%%%
\begin{definition}[Conservative Operator]
\label{def:Conservativity}
A monotone operator $M$ is said to be conservative if 
\begin{equation*}
\oint_{\mathit{\Gamma}}\left\langle M\alpha, \,d \alpha\right\rangle=0,
\end{equation*}
for every closed polyline $\mathit{\Gamma}\subset\dom M$ with finite length.
\end{definition}
%%%%%%%%%%%%%%%%%%%%%%%%%%%%%%%%%%%%%%

%%%%%%%%%%%%%%%%%%%%%%%%%%%%%%%%%%%%%%
\begin{lemma}%[Theorem~4.6 of {\citealp{romano1993potential}}]
\label{cyclically-monotone-conservative}
Let $M\colon\! E\rightrightarrows E^*$ be a multivalued map, then 

\textup{(1)} $M$ is cyclically monotone $\Longrightarrow M$ is conservative; 

\textup{(2)} $M$ is conservative and $\dom M$ is convex $\Longrightarrow M$ is cyclically monotone. 
\end{lemma}
%%%%%%%%%%%%%%%%%%%%%%%%%%%%%%%%%%%%%%

%%%%%%%%%%%%%%%%%%%%%%%%%%%%%%%%%%%%%%
\begin{definition}[Potential]
\label{def:Potential}
Let $C\neq\varnothing$ be a closed convex subset of $E$, and let $M\colon\! C\rightrightarrows E^*$ be a conservative operator. $\varphi$ is a potential of $M$ if 
\begin{equation*}
\begin{aligned}
\varphi\left(x\right)-\varphi\left(x_0\right)=\int_{\gamma\left(x_0,x\right)} \left\langle M\alpha, \,d\alpha\right\rangle, \quad\forall x, x_0\in C, \quad\text{and}\quad
\varphi\left(y\right)=+\infty,\quad\forall y\notin C,
\end{aligned}
\end{equation*}
where $\gamma\left(x_0,x\right)$ denotes an arbitrary finite-length oriented polyline in $E$ from $x_0$ to $x$.
\end{definition}
%%%%%%%%%%%%%%%%%%%%%%%%%%%%%%%%%%%%%%

%%%%%%%%%%%%%%%%%%%%%%%%%%%%%%%%%%%%%%
\begin{lemma}
\label{Conservative-convex}
Let $C\neq\varnothing$ be a closed convex subset of $E$, and let $M\colon\! C\rightrightarrows E^*$ be a conservative operator. If $\varphi$ is a potential of $M$, then $\varphi$ is convex and lower semicontinuous. 
\end{lemma} 
%%%%%%%%%%%%%%%%%%%%%%%%%%%%%%%%%%%%%%

%%%%%%%%%%%%%%%%%%%%%%%%%%%%%%%%%%%%%%
\begin{proof}
$\forall x, x_0\in C$, let $\gamma\left(x_0,x\right)$ be an arbitrary finite-length oriented polyline in $C$ from $x_0$ to $x$. 
According to \cref{integral-sup}, we have 
\begin{equation*}
\begin{aligned}
\varphi\left(x\right)-\varphi\left(x_0\right)=\int_{\gamma\left(x_0,x\right)} \left\langle M\alpha, \,d\alpha\right\rangle=\sup_{n,x_i,x_i^*}\left(\left\langle x_{n}^*, x-x_{n}\right\rangle+\sum_{i=0}^{n-1}\left\langle x_{i}^*, x_{i+1}-x_{i}\right\rangle\right).
\end{aligned}
\end{equation*}
The subsequent proof is similar to that of \cref{subdifferential-convexity}. 
Indeed, for each fixed tuple $\left(n,x_i,x_i^*\right)$, the function $x\mapsto\left\langle x_{n}^*, x-x_{n}\right\rangle+\sum_{i=0}^{n-1}\left\langle x_{i}^*, x_{i+1}-x_{i}\right\rangle$ is linear (and thus convex and lower semicontinuous). 
It follows that the superior envelope of these functions is convex and lower semicontinuous. 
Note that $C$ is closed and convex, 
therefore we obtain that $\varphi$ is convex and lower semicontinuous. 
\end{proof}
%%%%%%%%%%%%%%%%%%%%%%%%%%%%%%%%%%%%%%

%%%%%%%%%%%%%%%%%%%%%%%%%%%%%%%%%%%%%%
\begin{lemma}
\label{Potential}
Let $C\neq\varnothing$ be a convex subset of $E$. 
If $\dom\varphi=\dom\partial\varphi=C$, then $\partial\varphi$ is conservative, and the potential of $\partial\varphi$ can be chosen as $\varphi$, i.e., 
\begin{equation*}
\begin{aligned}
\varphi\left(x\right)-\varphi\left(z\right)=\int_{\gamma\left(z,x\right)}\left\langle\partial\varphi\left(\alpha\right),\,d\alpha\right\rangle, \quad\forall x,z\in C,
\end{aligned}
\end{equation*}
where $\gamma\left(z,x\right)$ denotes an arbitrary finite-length oriented polyline in $C$ from $z$ to $x$. 
\end{lemma} 
%%%%%%%%%%%%%%%%%%%%%%%%%%%%%%%%%%%%%%

%%%%%%%%%%%%%%%%%%%%%%%%%%%%%%%%%%%%%%
\begin{proof}
The proof of \cref{Potential} is derived from the following steps, 
\begin{align*}
\dom\varphi=\dom\partial\varphi=C \,
%\xLongrightarrow{~\hypertarget{step1}{\text{step 1}}~}~& \varphi \textup{~is convex and lower semicontinuous } \\
\xLongrightarrow{~\hypertarget{step1}{\text{step 1}}~}~& \partial\varphi\colon\! C\rightrightarrows E^* \textup{~is cyclically monotone} \\
\xLongleftrightarrow{~\hypertarget{step2}{\text{step 2}}~}~& \partial\varphi\colon\! C\rightrightarrows E^* \textup{~is conservative} \\
\xLongrightarrow{~\hypertarget{step3}{\text{step 3}}~}~& \textup{the potential of~} \partial\varphi \textup{~can be chosen as~} \varphi.
\end{align*}
%%%%%%%%%%%%%%%%%%%%%%%%%%%%%%%%%%%%%%

%%%%%%%%%%%%%%%%%%%%%%%%%%%%%%%%%%%%%%
The proof of \hyperlink{step1}{step 1} is as follows. $\forall n\geqslant 2$, $\forall \left(x_i,x_i^*\right)\in \gra\partial\varphi$, $i=1, 2, \cdots, n$, $x_{n+1}\equiv x_1$, 
\begin{equation*}
\varphi\left(x_{i+1}\right)-\varphi\left(x_{i}\right)\geqslant\left\langle x_i^*, x_{i+1}-x_i\right\rangle,
\end{equation*}
thus, 
\begin{equation*}
\sum_{i=1}^{n}\left\langle x_i^*, x_{i+1}-x_i\right\rangle\leqslant\sum_{i=1}^{n}\varphi\left(x_{i+1}\right)-\varphi\left(x_{i}\right)=0.
\end{equation*}
%%%%%%%%%%%%%%%%%%%%%%%%%%%%%%%%%%%%%%

%%%%%%%%%%%%%%%%%%%%%%%%%%%%%%%%%%%%%%
\hyperlink{step2}{Step 2} holds directly from \cref{cyclically-monotone-conservative}.
%%%%%%%%%%%%%%%%%%%%%%%%%%%%%%%%%%%%%%

%%%%%%%%%%%%%%%%%%%%%%%%%%%%%%%%%%%%%%
The proof of \hyperlink{step3}{step 3} relies on the following lemmas.
%%%%%%%%%%%%%%%%%%%%%%%%%%%%%%%%%%%%%%

%%%%%%%%%%%%%%%%%%%%%%%%%%%%%%%%%%%%%%
\begin{lemma}[Lebesgue.~Theorem~4.4 of {\citealp{komornik2016lectures}}]
\label{Lebesgue}
Every non-increasing or non-decreasing function $F\colon\!\left[a, b\right]\rightarrow\mathbb{R}$ is a.e. differentiable.
\end{lemma} 
%%%%%%%%%%%%%%%%%%%%%%%%%%%%%%%%%%%%%%

%%%%%%%%%%%%%%%%%%%%%%%%%%%%%%%%%%%%%%
\begin{lemma}[Lebesgue-Vitali.~Theorem~6.5 of {\citealp{komornik2016lectures}}]
\label{Lebesgue-Vitali}
Let $f\colon\!\left[a, b\right]\rightarrow\mathbb{R}$. If $F\colon\!\left[a, b\right]\rightarrow\mathbb{R}$ is absolutely continuous, has bounded variation, and $F'=f$ a.e. Then $f$ is integrable and 
\begin{equation*}
\int_a^b f\left(x\right)\,d x=F\left(b\right)-F\left(a\right).
\end{equation*}
\end{lemma} 
%%%%%%%%%%%%%%%%%%%%%%%%%%%%%%%%%%%%%%

%%%%%%%%%%%%%%%%%%%%%%%%%%%%%%%%%%%%%%
Let 
\begin{equation*}
\begin{aligned}
F\left(\lambda\right)=\varphi\left(z+\lambda\left(x-z\right)\right),\quad 
f\left(\lambda\right)=\left\langle\partial\varphi\left(z+\lambda\left(x-z\right)\right), x-z\right\rangle,\quad\lambda\in\left[0, 1\right],\quad\forall x, z\in C,
\end{aligned}
\end{equation*}
then $F$ is convex on $\left[0, 1\right]$, that is, $\exists \tau\in\left[0, 1\right]$, such that $F$ is non-increasing on $\left[0, \tau\right]$ and non-decreasing on $\left[\tau, 1\right]$. Thus, $F$ has bounded variation, and $F'=\partial F$ a.e. according to \cref{Lebesgue}. Since $\forall \lambda, \mu\in\left[0, 1\right]$, 
\begin{equation*}
\begin{aligned}
F\left(\mu\right)-F\left(\lambda\right)
&=\varphi\left(z+\mu\left(x-z\right)\right)-\varphi\left(z+\lambda\left(x-z\right)\right) \\
&\geqslant \left\langle\partial\varphi\left(z+\lambda\left(x-z\right)\right), x-z\right\rangle\left(\mu-\lambda\right)
=f\left(\lambda\right)\left(\mu-\lambda\right)
\end{aligned}
\end{equation*}
we have $f\left(\lambda\right)\subset\partial F\left(\lambda\right)$, and then $F'=f$ a.e.
%%%%%%%%%%%%%%%%%%%%%%%%%%%%%%%%%%%%%%

%%%%%%%%%%%%%%%%%%%%%%%%%%%%%%%%%%%%%%
Note that 
\begin{equation*}
\begin{aligned}
| F\left(\lambda\right)-F\left(\mu\right)|
\leqslant\max\left\{| f\left(\lambda\right)|, | f\left(\mu\right)|\right\} | \lambda-\mu |
\leqslant\max\left\{|\sup f\left(0\right)|, |\inf f\left(1\right)|\right\} | \lambda-\mu |, 
\end{aligned}
\end{equation*}
$F$ is Lipschitz continuous, then $F$ is absolutely continuous. According to \cref{Lebesgue-Vitali}, and note that $\partial\varphi$ is conservative, we have 
\begin{equation*}
\begin{aligned}
\int_{\gamma\left(z,x\right)}\left\langle\partial\varphi\left(\alpha\right),\,d\alpha\right\rangle
=\int_{\overrightarrow{z x}}\left\langle\partial\varphi\left(\alpha\right),\,d\alpha\right\rangle=\int_0^1 f\left(\lambda\right)\,d\lambda
=F\left(1\right)-F\left(0\right)=\varphi\left(x\right)-\varphi\left(z\right),
\end{aligned}
\end{equation*}
which completes the proof. 
\end{proof}
%%%%%%%%%%%%%%%%%%%%%%%%%%%%%%%%%%%%%%

%%%%%%%%%%%%%%%%%%%%%%%%%%%%%%%%%%%%%%
\end{document}